\documentclass{article}
\pdfoutput=1

\usepackage{arxiv}
\usepackage[T1]{fontenc}
\usepackage{lmodern}
\usepackage{anyfontsize}

\usepackage{pifont}
\providecommand{\mg}{m g}

\newcommand{\equidexflow}{\textup{\textsc{EquiDexFlow}}}
\newcommand{\shorttitle}{{\equidexflow}: SE(3)-Equivariant Dexterous Grasp Generative Flows}
\title{{\equidexflow}: Contact-Grounded SE(3)-Equivariant Dexterous Grasp Generative Flows}

\author{
   Clinton Enwerem,\, John S. Baras,\, Calin Belta \\Institute for Systems Research, \\University of Maryland, College Park, MD, U.S.A.\\ {Emails: \{\texttt{enwerem, baras, calin}\}@umd.edu}
}

\renewcommand{\today}{}

\usepackage{amsmath,amssymb,amsfonts,amsthm,mathtools}
\usepackage{graphicx}
\usepackage{booktabs}
\usepackage[para]{footmisc}
\usepackage[dvipsnames]{xcolor}
\usepackage{colortbl}
\usepackage{siunitx}
\usepackage[colorlinks=true,allcolors=blue,urlcolor=blue]{hyperref}
\usepackage{bookmark}
\usepackage{cleveref}
\newtheorem{remark}{Remark}
\newtheorem{proposition}{Proposition}
\newtheorem{theorem}{Theorem}
\newtheorem{definition}{Definition}
\newtheorem{problem}{Problem}
\usepackage{caption}
\usepackage{subcaption}
\usepackage{makecell}
\usepackage{wrapfig}
\usepackage{enumitem}
\usepackage{url}
\usepackage{titlesec}

\renewcommand{\thesection}{\Roman{section}}
\renewcommand{\thesubsection}{\Alph{subsection}}
\renewcommand{\thesubsubsection}{\arabic{subsubsection})}
\titleformat{\section}[block]
  {\normalfont\normalsize\scshape\centering}
  {\thesection.}{0.55em}{}
\titleformat{\subsection}[block]
  {\normalfont\itshape}
  {\thesubsection.}{0.5em}{}
\titleformat{\subsubsection}[runin]
  {\normalfont\itshape}
  {\thesubsubsection}{0.35em}{}[:]
\titlespacing*{\section}{0pt}{8pt plus 1pt minus 1pt}{3pt plus 1pt minus 1pt}
\titlespacing*{\subsection}{0pt}{5pt plus 1pt minus 1pt}{1pt plus 1pt minus 1pt}
\titlespacing*{\subsubsection}{0pt}{3pt plus 1pt minus 1pt}{0.5em}
\titlespacing*{\paragraph}{0pt}{2pt plus 1pt minus 1pt}{0.5em}

\expandafter\def\expandafter\normalsize\expandafter{%
    \normalsize
    \setlength{\abovedisplayskip}{3pt}
    \setlength{\belowdisplayskip}{3pt}
    \setlength{\abovedisplayshortskip}{3pt}
    \setlength{\belowdisplayshortskip}{3pt}
}

\newcommand{\R}{\mathbb{R}}
\newcommand{\G}{\bar{G}}
\newcommand{\SE}{\operatorname{SE}}
\newcommand{\SO}{\operatorname{SO}}
\newcommand{\dataset}{\mathcal{D}}
\newcommand{\graspmap}{\mathcal{G}}

\Crefformat{figure}{Figure~#2#1#3}
\Crefformat{table}{Table~#2#1#3}
\Crefformat{section}{Section~#2#1#3}
\Crefformat{subsection}{Section~#2#1#3}
\Crefformat{equation}{(#2#1#3)}
\Crefformat{proposition}{Proposition~#2#1#3}
\Crefformat{theorem}{Theorem~#2#1#3}
\Crefformat{remark}{Remark~#2#1#3}

\makeatletter
\renewcommand{\p@subsection}{\thesection.}
\renewcommand{\p@subsubsection}{\thesection.\Alph{subsection}.}
\makeatother

\definecolor{cRowOurs}{rgb}{0.886,0.945,0.918}
\definecolor{cRowOff}{rgb}{0.976,0.910,0.894}

\begin{document}
\maketitle

\begin{figure}[htb]
\centering
\vspace{-20pt}
\includegraphics{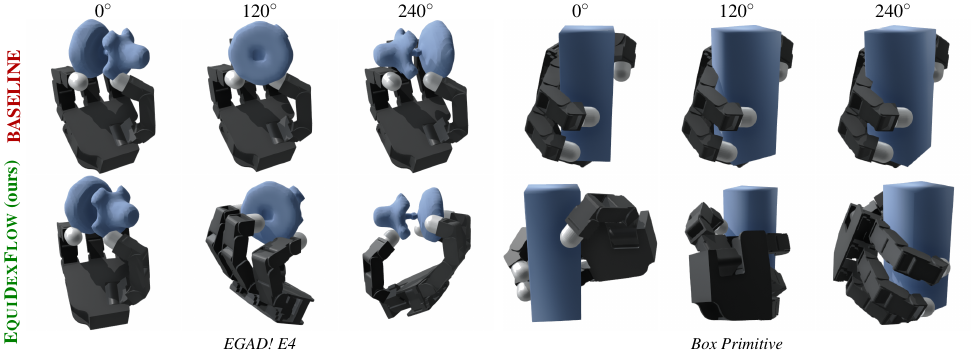}
\caption{\textbf{SE(3)-Equivariant Dexterous Grasps Generated by \equidexflow{}.}
EGAD!~\cite{morrisonEGADEvolvedGrasping2020b} E4 (\emph{left three}) and a box primitive (\emph{right three}) under \(0^\circ\), \(120^\circ\), and \(240^\circ\) rotations about the vertical axis.
\emph{Top}: naively applying the canonical grasp to a rotated object leaves fingertips off the surface or causes severe penetration. \emph{Bottom}: \equidexflow{} co-rotates the wrist pose and finger
configuration with the object, preserving contact placement under \(\SE(3)\) action.}
\label{fig:teaser}
\end{figure}

\begin{abstract}
Most learned dexterous grasp generators relegate contact forces to a downstream verification step, so a kinematically-plausible pose can still violate the conditions for a stable physical grasp. We address this with \equidexflow{}, an SE(3)-equivariant flow-matching model that jointly predicts wrist pose, joint angles, fingertip contacts, surface normals, and contact forces from an object point cloud. Our architecture projects contacts onto the object surface and forces into the Coulomb friction cone by construction, so placement and friction compliance hold without loss penalties. We prove end-to-end SE(3) equivariance and verify it empirically over 200 rotations, with wrist residuals below \(0.04^\circ\) and exactly zero joint deviation. Trained on 8,100 force-closure grasps across 81 objects for the 16-DoF Allegro Hand, our model achieves zero friction violations, the best composite score, and the lowest wrench residual among all ablation variants. We retarget decoded fingertip contacts to a 16-DoF LEAP Hand via per-finger inverse kinematics, and our hardware-feasible refinement places every joint at least \(5\%\) inside its actuator envelope while preserving wrench balance. On the physical robot, retargeted \equidexflow{}-decoded grasps complete open-loop pick-and-hold trials on all six test objects, with every asymmetric object succeeding at both the canonical pose and a \(120^\circ\) co-rotation. Videos, code, and checkpoints are available at \href{https://equidexflow.github.io}{\texttt{equidexflow.github.io}}.
\end{abstract}

\section{Introduction}\label{sec:intro}
To reliably grasp and manipulate an object, modern multi-fingered dexterous manipulation systems require not only a hand configuration that places the hand's fingertips on the manipuland's surface, but also a contact structure and force distribution that together resist gravity, satisfy friction constraints, and remain robust to small perturbations. Existing learned generators rely on architectures that decouple pose and force: the pose head learns without any notion of admissible force directions, while a downstream force optimizer assumes a contact geometry that deviates from the pose head's training-time preferences. As a consequence, such generators often produce geometrically plausible wrist poses that fail force closure or violate friction limits. 

We instead learn the coupling directly: given an object point cloud \(\mathcal{P}\), a single \(\theta\)-parameterized model captures the joint conditional distribution \(p_\theta(q_h, C, F \mid \mathcal{P})\) over the hand configuration \(q_h\), the contact set \(C\), and the contact force field \(F\). These components constrain one another, since the hand configuration determines reachable contacts, the contact geometry restricts admissible force directions, and the resulting forces decide physical stability. Modeling the three jointly lets a single forward pass produce grasps that are diverse, geometrically consistent, and physically grounded. We build on two complementary ideas following~\cite{lim2025equigraspflow}: \emph{SE(3) equivariance} (\Cref{fig:teaser}), achieved by parameterizing all representations with Vector-Neuron (VN) layers~\cite{deng2021vn} that preserve \(\SO(3)\) structure throughout the learning pipeline, and \emph{flow matching}~\cite{lipman2023flow} on the Lie group \(\SE(3)\), which formulates wrist pose generation as an ODE flow integrated with a Munthe-Kaas RK4 solver~\cite{munthekaas1998rungekutta} rather than diffusion or score matching~\cite{song2021scorebased}.

The resulting system, \emph{\equidexflow}, takes an object point cloud and a kinematic model of a \(D\)-DoF, \(M\)-fingered robotic hand and produces, in a single forward pass, a wrist \(\SE(3)\) pose, \(D\) joint angles from a conditional normalizing flow, and \(M\) fingertip contacts projected onto the object's surface, with contact forces projected into the friction cone, all jointly consistent with the learned distribution. Trained on force-closure-certified grasps from FRoGGeR~\cite{li2023frogger}, with physics losses enforcing wrench balance, friction feasibility, and collision avoidance, \equidexflow{} generates \(K\) grasp candidates ranked by a composite quality score.\\[-1em] 

\textbf{Contributions:}
We summarize the paper's main contributions as follows.
\begin{enumerate}[leftmargin=*, labelsep=0.5em]
\item \textbf{Force-Aware Grasp Generation:} \equidexflow{} jointly predicts hand configuration, fingertip contacts, and contact forces from a point cloud in one forward pass, without post-hoc recovery.
\item \textbf{Architectural Physics Guarantees:} We introduce two physics-preserving mechanisms in \equidexflow{}'s decoder heads (\Cref{ssec:archov}): a cone projection in the force decoder that keeps every decoded force inside the Coulomb friction cone, and a differentiable surface projection in the contact decoder that keeps contacts on the object's surface. Both guarantees hold by construction for every \equidexflow{} ablation variant.
\item \textbf{End-to-End SE(3) Equivariance:} We prove that \equidexflow{}'s pipeline is SE(3)-equivariant by construction (\Cref{thm:equiv}) and verify the result empirically: a wrist rotation residual under \(0.04^\circ\) and zero joint deviation over 200 \(\SO(3)\) rotations (\Cref{tab:equivariance_binned}).
\item \textbf{Curated-Dataset Validation:} We train \equidexflow{} on 8{,}100 force-closure grasps across 81 objects and evaluate on 811 held-out test grasps, recording the best composite quality score, lowest wrench residual, and zero friction violations across all ablation variants.
\end{enumerate}

\subsection{Related Work}\label{sec:related}
Dexterous grasp synthesis has progressed from analytical force-closure search~\cite{liu2021synthesizing,li2023frogger} and large validated corpora~\cite{wang2023dexgraspnet,zhang_dexgraspnet_2024} toward learned generators. Contact-centric methods replace direct pose prediction with intermediate representations that transfer across embodiments~\cite{li2023gendexgrasp,zhao2024graingrasp}. However, they do not model the coupled contact-force distribution. Pose-centric generative work spans VAEs~\cite{mousavian2019sixdofgraspnet} and SE(3) diffusion~\cite{urain2023se3diffusionfields} for parallel-jaw grippers, equivariant 6-DoF grip synthesis~\cite{weng2024capgrasp}, SE(3)-equivariant continuous flows~\cite{lim2025equigraspflow}, and dexterous-hand families using diffusion~\cite{weng2024dexdiffuser}, flow matching~\cite{feng2025ffhflow}, or conditional flows for multimodal joints~\cite{xu2023unidexgrasp}. Equivariance itself has been studied for planar~\cite{zhu2022sample} and graph-based~\cite{huang2023edge} grasping, and for Euclidean~\cite{kohler2020equivariant} and Riemannian~\cite{katsman2021equivariant,chen2023equidiff,rozenberg2023semiequivariant} flows. Yet none of these studies treats contact forces as a first-order generative variable. A separate physics-injecting thread brings forces closer to grasp generation through inference-time energy guidance~\cite{zhao2026effgrasp}, force-map reranking~\cite{semenyakina2026graspsense}, and adversarial synthesis under contact and belief uncertainty~\cite{chen_adversarial_2025,enweremRiskConstrainedBeliefSpaceOptimization2026a,enweremVariationalNeuralParameterizations2026a}. Even there, forces enter downstream rather than as variables learned jointly with the pose and contacts. We unify these threads with an equivariant generator that learns hand configuration, contacts, and forces jointly. To set up the equivariance problem, we first review Lie-group flow matching and Vector Neurons in the next section. For a compact but thorough Lie theory refresher, we recommend~\cite{sola_micro_2021}.

\subsection{Mathematical Preliminaries}\label{sec:prelim}
We recall the three building blocks our generator rests on: the geometry of the special Euclidean group, the formulation of flow matching when the carrier manifold is a Lie group, and the Vector-Neuron layers that carry \(\SO(3)\) structure through the network.

\textbf{The Special Euclidean Group \(\boldsymbol{\SE(3)}\):}
The special Euclidean group \(\SE(3) = \SO(3) \ltimes \R^3\) is the group of rigid-body transformations in three dimensions. Each element \(T = (R, x)\), with \(R \in \SO(3)\) and \(x \in \R^3\), acts on a point \(p \in \R^3\) via \(T \cdot p = Rp + x\)~\cite{murrayMathematicalIntroductionRobotic2017a}. Composition is \(T_1 T_2 = (R_1 R_2,\; R_1 x_2 + x_1)\). The Lie algebra \(\mathfrak{se}(3) \cong \R^6\) consists of twist vectors \(\xi = (\omega, v)\), where \(\omega \in \R^3\) is an angular velocity and \(v \in \R^3\) is a linear velocity. The exponential map \(\exp\colon \mathfrak{se}(3) \to \SE(3)\) and logarithmic map \(\log\colon \SE(3) \to \mathfrak{se}(3)\) provide a local diffeomorphism between the group and its tangent space at the identity.

\textbf{Flow Matching on Lie Groups:}\label{ssec:flowmatchlie}
Flow matching~\cite{lipman2023flow} learns a time-dependent velocity field \(v_\theta(x, t)\) that transports a simple source distribution \(p_0\) to a target distribution \(p_1\) via an ODE, \(\dot{x}_t = v_\theta(x_t, t)\). In Euclidean space, the conditional flow between a source sample \(x_0\) and target \(x_1\) is the linear interpolant, \(x_t = (1 - t)x_0 + t x_1\), with constant velocity \(u_t = x_1 - x_0\). On a Lie group \(G\), the geodesic \(g_t = g_0 \exp(t \cdot \log(g_0^{-1} g_1))\) replaces the linear interpolant, and the velocity lives in the Lie algebra \(\mathfrak{g}\)~\cite{chen2023riemannian}. The training objective remains an \(L^2\) regression, \(\mathcal{L}_{\mathrm{flow}} = \|v_\theta(g_t, t) - u_t\|^2\), where \(u_t \in \mathfrak{g}\) is the analytical geodesic tangent. At inference, we integrate with a Lie-group ODE solver such as the Munthe-Kaas Runge-Kutta (RK) scheme~\cite{munthekaas1998rungekutta}, which performs classical RK steps in the Lie algebra and maps back to the group through the exponential map, preserving the group structure at every step.

\textbf{\texorpdfstring{Vector Neurons and \(\boldsymbol{\SO(3)}\) Equivariance}{Vector Neurons and SO(3) Equivariance}:}
Vector Neurons~\cite{deng2021vn} replace each scalar feature channel with a 3-vector, so a feature tensor has shape \((C, 3)\) rather than \((C,)\). A VN-Linear layer mixes channels via a shared weight matrix without touching the spatial axis, giving \(z' = Wz\) with \(W \in \R^{C' \times C}\) and \(z \in \R^{C \times 3}\). Because \(W\) acts only on the channel axis, rotating the input by \(R \in \SO(3)\) commutes with the linear map: \(W(zR^\top) = (Wz)R^\top\). VN layers are therefore \(\SO(3)\)-equivariant by construction, and nonlinearities applied per-channel through vector norms or learned direction projections preserve the property. Stacking VN layers into a DGCNN backbone~\cite{deng2021vn} yields an encoder whose output \(z_O\) transforms as \(z_O \mapsto z_O R^\top\) under input rotation, providing the equivariant representation our generator builds on.

\section{Problem Formulation}\label{sec:problem}
We first set notation for kinematic and structured grasps, then state the learning problem the rest of the paper attacks.
We represent a hand-centered grasp as \(G=(T_w,q_h)\in\SE(3)\times\mathcal{Q}_h\), where \(\mathcal{Q}_h\subset\R^{D}\) is the space of joint configurations of a \(D\)-DoF hand, \(T_w=(R_w,x_w)\) is the wrist pose, and \(q_h\) is the hand configuration. Denote the hand geometry as \(\mathcal{H}(G):=T_w\,\mathcal{H}_{\mathrm{local}}(q_h)\), where \(\mathcal{H}_{\mathrm{local}}(q_h)\) is the set of hand link surfaces in the wrist frame, computed via forward kinematics from \(q_h\). Under a rigid transform \(A:=(R_A,x_A)\in\SE(3)\) acting on \(G\) as \(A\cdot G=(AT_w,q_h)\), forward kinematics is equivariant, i.e., 
\begin{equation}\label{eq:fk_equiv}
\mathcal{H}(A\cdot G)=A\cdot\mathcal{H}(G),
\end{equation}
so wrist-pose equivariance suffices for the induced hand mesh when \(q_h\) is treated as an invariant shape coordinate. From here on we generate grasps in the object frame. Let \(\smash{{}^{O}\mathcal{P}}\) be the object point cloud, and define the object-relative grasp \(\smash{{}^{O}G:=({}^{O}T_w,q_h)}\), whose wrist pose lives in an object-centered frame anchored at the object's center of mass. Conditioned on \(\smash{{}^{O}\mathcal{P}}\), our grasp generator is a joint likelihood over the relative grasp, fingertip contacts, and the contact geometry they induce,
\begin{equation}\label{eq:obj_frm_mod}
p_\theta({}^{O}G,{}^{O}C,{}^{O}\Lambda \mid {}^{O}\mathcal{P}),
\end{equation}
where \(\smash{{}^{O}C=\{({}^{O}p_i,{}^{O}n_i,B_i,\ell_i)\}_{i=1}^{M}}\) collects \(M\) fingertip contacts in the object frame (one per finger), \(\smash{B_i=[t_{i,1},t_{i,2},{}^{O}n_i]\in\SO(3)}\) is the local contact frame, \(\ell_i\) is the corresponding finger label, and \(\smash{{}^{O}\Lambda=\{\alpha_i\}_{i=1}^{M}}\) are the local force coordinates. Dropping the object-frame labels, the world-frame analog of \eqref{eq:fk_equiv} is the distributional condition
\begin{equation}
G \sim p_\theta(\cdot \mid \mathcal{P})
\quad\Longrightarrow\quad
A\cdot G \sim p_\theta(\cdot \mid A\mathcal{P}).
\label{eq:dist_equi}
\end{equation}
In the object frame, the distributional requirement of \eqref{eq:dist_equi} becomes invariance of the object-relative grasp distribution to the object's world pose encoded in \(\mathcal{P}\). We therefore retain the object-centered frame convention throughout the remainder of this paper.

\begin{definition}[SE(3)-Equivariant Grasp Generator]\label{def:eq_generator}
Let \(\G=(T_w,q_h,C,N,F)\) denote a structured grasp extending
the kinematic grasp \(G=(T_w,q_h)\) with the wrist pose \(T_w\in\SE(3)\), invariant hand joint vector \(q_h\), contact data \(C=\{p_i,B_i\}_{i=1}^{M}\), per-contact normals \(N=\{n_i\}_{i=1}^{M}\), and per-contact Cartesian forces, \(F=\{f_i\}_{i=1}^{M}\). The component-wise action of \(A=(R_A,x_A)\in\SE(3)\) is
\[
A\cdot \G=\bigl(AT_w,\;q_h,\;\{R_Ap_i+x_A,\;R_AB_i\},\;\{R_An_i\},\;\{R_Af_i\}\bigr).
\]
A conditional grasp generator \(p_\theta(\cdot\mid\mathcal{P})\) is \emph{SE(3)-equivariant} if it satisfies the distributional condition \eqref{eq:dist_equi}, i.e., \(\G\sim p_\theta(\cdot\mid\mathcal{P})\Rightarrow A\cdot \G\sim p_\theta(\cdot\mid A\mathcal{P})\) for every \(A\in\SE(3)\).
\end{definition}

\begin{problem}[Equivariant Force-Aware Grasp Generation]\label{prob:equidex}
Given an object point cloud \(\smash{{}^{O}\mathcal{P}}\) and a \(D\)-DoF, \(M\)-fingered hand, learn a parameterized distribution
\(p_\theta(\G \mid \smash{{}^{O}\mathcal{P}})\)
over the structured grasp \(\G=(T_w,q_h,C,N,F)\) such that
\begin{enumerate}[leftmargin=*, label=(\roman*)]
  \item \emph{SE(3) equivariance:} \(p_\theta\) satisfies \Cref{def:eq_generator};
  \item \emph{Physical feasibility:} every sample admits per-contact forces \(f_i\) inside the local Coulomb friction cone and a contact wrench that balances gravity at the object's center of mass.
\end{enumerate}
The model is trained from a simulated dataset of force-closure-certified grasps and evaluated at inference both on its distributional properties \emph{and} on the executability of the top-ranked sample on a physical robot.
\end{problem}

Spatial quantities (\(T_w,p_i,n_i,B_i,f_i\)) are covariant under rigid motion, while \(q_h\) and local force coordinates \(\alpha_i=B_i^\top f_i\) are intrinsic to the hand or contact frame and therefore invariant (see \cite{murrayMathematicalIntroductionRobotic2017a}, \S5 and \cite{sola_micro_2021}). This separation enables a single equivariant generator to produce physically consistent forces (\Cref{prop:norm_cons,prop:force_eq} and \Cref{thm:equiv}). The cone projection enforces friction compliance only relative to its construction frame, and a different normal at evaluation time can invalidate the bound (\Cref{prop:norm_cons}). Physical feasibility is not implied by \eqref{eq:fk_equiv}. We evaluate physical quality through the object-frame residual wrench \(\smash{{}^{O}r_w}\) \eqref{eq:obj_wren_res}, penalized by \(\mathcal{L}_w\) \eqref{eq:loss_wrench}, with friction in local coordinates as \(\mathcal{L}_\mu\) \eqref{eq:loss_friction}. These terms are coordinate-invariant but not invariant to physical rotations under fixed gravity, so the gravity vector \(\smash{{}^{O}g}\) enters as physical conditioning. For execution, the object pose \(\smash{{}^{W}T_O}\) induces a wrist target \(\smash{{}^{W}T_w = {}^{W}T_O\,{}^{O}T_w}\) and an arm-feasible set \(\mathcal{A}\) \eqref{eq:rb_feas_set}. This stage is not equivariant because the robot base fixes a world frame. The candidate score combines object-relative quality with robot feasibility,
\begin{equation}
J = \beta_1 Q_{\mathrm{geom}}({}^{O}G)
 + \beta_2 Q_{\mathrm{phys}}({}^{O}C,{}^{O}\Lambda)
 + \beta_3 Q_{\mathrm{task}}\!\bigl(\mathcal{A}({}^{O}G;{}^{W}T_O,\mathcal{E})\bigr)
 - \beta_4 Q_{\mathrm{risk}}({}^{O}G).
\label{eq:ranking}
\end{equation}
We train on a simulated dataset \(\smash{\mathcal{D}=\{({}^{O}\mathcal{P}^{(n)},{}^{O}G^{(n)},{}^{O}C^{(n)},{}^{O}\Lambda^{(n)})\}_{n=1}^{N}}\). The training objective \eqref{eq:equidexflowloss} couples flow matching, contact, force, and physical losses so that geometry, contacts, and forces obey the correct transformation laws. Ranking and robot feasibility handle execution constraints that are not task symmetries.

\section{Methodology}\label{sec:method}
We seek a generative model \(p_\theta(\G \mid \mathcal{P})\) over the structured grasp of \Cref{def:eq_generator}, with five outputs: wrist pose \(T_w\in\SE(3)\), finger joint angles \(q_h\in\mathcal{Q}_h\subset\R^{D}\), per-finger contact positions \(C\in\R^{M\times 3}\), inward surface normals \(N\in\R^{M\times 3}\), and contact forces \(F\in\R^{M\times 3}\). The point cloud \(\mathcal{P}\in\R^{3\times N}\) conditions all five, and the hand contributes \(M\) fingertips and \(D\) actuated joints (\Cref{sec:problem}). The two requirements of \Cref{prob:equidex} drive the design: \emph{SE(3) equivariance} (rotating the object rotates the generated grasp identically, without retraining) and \emph{physical feasibility} (every predicted contact force lies inside the Coulomb friction cone, and the aggregate balances gravity). Rather than enforce feasibility through auxiliary loss penalties alone, \equidexflow{} projects predicted forces onto the friction cone by construction, so the output is physically admissible under the projection contact frame (\Cref{prop:norm_cons}).

\subsection{Architecture Overview}\label{ssec:archov}
Given an object point cloud \(\mathcal{P}\in\R^{3\times N}\) and a dexterous hand with \(M\) fingertips, \equidexflow{} instantiates \eqref{eq:obj_frm_mod} as a single forward pass with five components (\Cref{fig:architecture}): (1)~a Vector-Neuron Dynamic Graph CNN (VN-DGCNN) point-cloud encoder that extracts SO(3)-equivariant object features \(z_O\in\R^{C\times 3}\), (2)~an SE(3) flow backbone that generates the wrist pose by integrating an equivariant velocity field from a source distribution to the target, (3)~a contact decoder that predicts per-finger contact positions \(\hat{C}\in\R^{M\times 3}\), (4)~a normal decoder that predicts per-finger inward surface normals \(\hat{N}\in\R^{M\times 3}\), and (5)~a force decoder that predicts contact forces
\(\hat{F}\in\R^{M\times 3}\) projected into the Coulomb friction cone by construction, together with a hand-joint flow decoder that samples finger joint angles \(\hat{q}_h\) from a learned conditional distribution. The hand embodiment enters only through \(M\) and the joint-angle dimensionality. The architecture, losses, and equivariance guarantees are embodiment-agnostic.
\begin{figure}[t]
\centering
\includegraphics[width=0.8\columnwidth]{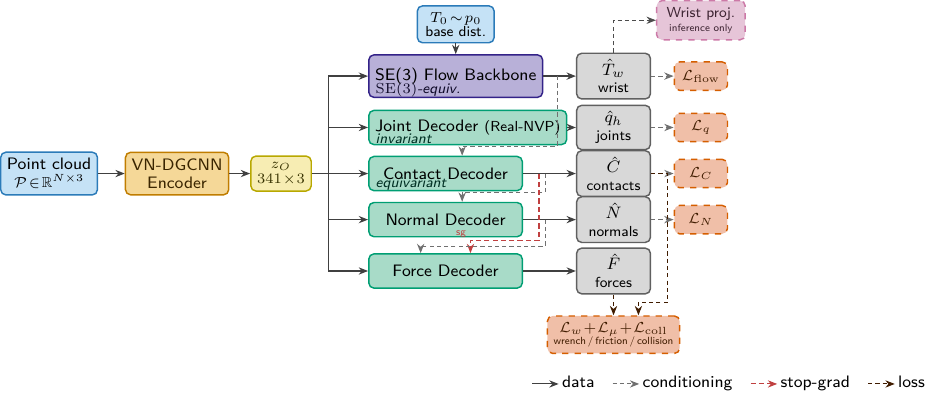}
\caption{\textbf{\equidexflow{} Architecture.} A VN-DGCNN encoder produces equivariant features \(z_O\) that drive five heads: an SE(3) wrist-pose flow and joint, contact, normal, and force decoders. The normal decoder replaces centroid estimates that shrink the wrench-balance feasible set, and a cone-projection layer maps each force into the friction cone by construction. Wrench, friction, and collision losses supervise the contact and force outputs.}
\label{fig:architecture}
\end{figure}
The joint distribution in \eqref{eq:obj_frm_mod} factorizes into two stochastic components and three deterministic decoder heads, all conditioned on the shared equivariant features \(z_O\). The generative process is
\begin{subequations}\label{eq:factorization}
\begin{align}
  T_0 \sim p_0,\quad T_w &= \Phi_\theta^{1\leftarrow 0}(T_0;\, z_O), \label{eq:fac_flow} \\
  q_h &\sim p_\theta^q(\cdot \mid z_O, T_w), \label{eq:fac_joints} \\
  C &= g_\theta^C(z_O,\, T_w), \label{eq:fac_contact} \\
  N &= \eta_\theta(z_O,\, C), \label{eq:fac_normal} \\
  \Lambda &= \Pi_\mu\!\bigl(\tilde{f}_\theta(z_O,\, C),\; N\bigr), \label{eq:fac_force}
\end{align}
\end{subequations}
where \(\Phi_\theta^{1\leftarrow 0}\) is the time-one map of the ODE \(\dot{T}_t = v_\theta(T_t, z_O, t)\) on \(\SE(3)\), \(p_0\) is a source distribution on \(\SE(3)\), \(p_\theta^q\) is the conditional joint-angle distribution modeled by a normalizing flow (\S\ref{sssec:hand_joint_decoder}), \(g_\theta^C\) and \(\eta_\theta\) are the contact-position and surface-normal decoders, \(\tilde{f}_\theta\) is the raw force decoder, and \(\Pi_\mu\) projects into the Coulomb friction cone. The induced joint density is
\begin{multline}\label{eq:joint_density}
p_\theta(T_w, q_h, C, N, \Lambda \mid \mathcal{P})
= p_\theta(T_w \mid \mathcal{P})\,
  p_\theta^q(q_h \mid z_O, T_w)\,
  \delta\!\bigl(C - g_\theta^C(z_O, T_w)\bigr)\\
  {}\cdot\delta\!\bigl(N - \eta_\theta(z_O, C)\bigr)\,
  \delta\!\bigl(\Lambda - \Pi_\mu(\tilde{f}_\theta(z_O, C),\, N)\bigr),
\end{multline}
where \(p_\theta(T_w \mid \mathcal{P}) = (\Phi_\theta^{1\leftarrow 0}(\cdot\,; z_O))_\# p_0\) is the push-forward through the learned flow map~\cite{lipman2023flow}. Stochasticity enters through the wrist flow and the joint-angle flow, while the remaining decoders are deterministic. The sequential contact-to-normal-to-force chain couples wrench balance with contact placement through the training loss. We expand each component below, from the equivariant point-cloud encoder through the flow-matching backbone to the contact, normal, and force decoders (details in \Cref{apx:arch_details}).

\textbf{Equivariant Point-Cloud Encoder:}\label{sssec:encoder}
A VN-DGCNN~\cite{deng2021vn} encodes \(\mathcal{P}\) into a global feature \(z_O\in\R^{341\times3}\) via seven \texttt{VNLinearLeakyReLU} layers with \(k\)-NN graph construction (widths and \(k\) in \Cref{tab:hyperparams}). Rotating the input by \(R\in\SO(3)\) yields \(z_O R^\top\), satisfying \eqref{eq:dist_equi}.

\textbf{SE(3) Flow Matching Backbone:}\label{sssec:flow_backbone}
We formulate wrist pose generation as flow matching~\cite{lipman2023flow,chen2018neural} on \(\SE(3)\)~\cite{chen2023riemannian,mathieu2020riemannian}, connecting a uniform source \(T_0\sim p_0\) to the data target via the geodesic path \(R_t = R_0\exp(t\log(R_0^{-1}R_1))\), \(x_t = x_0+t(x_1-x_0)\). We train a VN-MLP vector field \(v_\theta(T_t,z_O,t)\) with \(\mathcal{L}_{\mathrm{flow}}=\|v_\theta-u_t\|^2\) and integrate it at inference with a Munthe-Kaas RK4 solver~\cite{munthekaas1998rungekutta} under classifier-free guidance~\cite{ho2022cfg}.

\textbf{Contact Decoder:}\label{sssec:contact_decoder}
An equivariant path predicts four per-finger contact offsets anchored to the predicted wrist translation \(\hat{x}_w\), which a differentiable soft-nearest-neighbor projection snaps onto the input surface, so contacts lie on the object by construction.

\textbf{Normal Decoder:}\label{sssec:normal_decoder}
A dedicated equivariant decoder predicts per-finger inward surface normals \(\hat{n}_i\), decoupling normal estimation from force optimization so the wrench-balancing feasible set is not locked to a centroid heuristic.

\textbf{Force Decoder:}\label{sssec:force_decoder}
The force decoder predicts per-finger contact forces directly in the object frame using \(\SO(3)\)-equivariant layers with per-finger contact conditioning. For each finger~\(i\), it concatenates the global features \(z_O\in\R^{341\times3}\) with the predicted contact \(\hat{c}_i\) as an equivariant vector channel (\(z_i\in\R^{342\times3}\)) and passes \(z_i\) through a shared VN stack (widths in \Cref{tab:hyperparams}) to produce a raw force vector \(\tilde{f}_i\in\R^3\). The contact channel provides per-finger spatial grounding. Without it the decoder collapses to a uniform radial-push solution. A cone-projection layer then maps each raw force \(\tilde{f}_i\) into the Coulomb friction cone by construction,
\begin{equation}
  f_{n,i} = \mathrm{softplus}(\tilde{f}_i \cdot \hat{n}_i), \quad
  f_{t,i} = \tilde{f}_i - (\tilde{f}_i \cdot \hat{n}_i)\,\hat{n}_i, \quad
  \hat{f}_i = f_{n,i}\,\hat{n}_i + \min\!\bigl(1,\, \mu f_{n,i} / \|f_{t,i}\|\bigr)\, f_{t,i},
\end{equation}
where \(\hat{n}_i\) is the learned inward surface normal from the normal decoder (\S\ref{sssec:normal_decoder}). \Cref{fig:contact_frame} illustrates the geometry. The softplus ensures compressive normal force (\(f_{n,i} \geq 0\)) and the tangential scaling enforces \(\|f_{t,i}\| \leq \mu f_{n,i}\), so friction-cone compliance is an architectural guarantee rather than a loss-driven objective. Decomposing the force in the local contact frame is also what makes the reconstructed Cartesian force end-to-end equivariant (\Cref{prop:force_eq}).

\begin{figure}[htb]
  \centering
 \includegraphics[width=0.85\columnwidth]{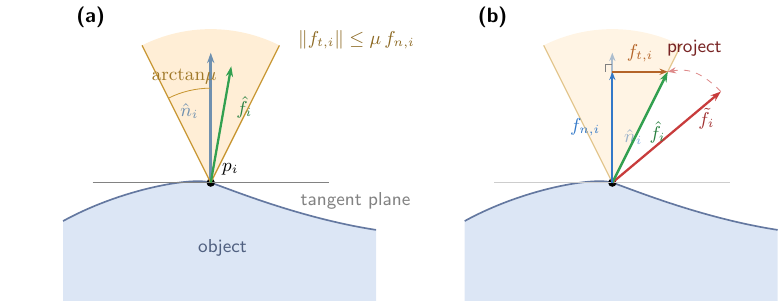}
  \caption{\textbf{Contact Geometry and Friction Cone Projection.} (a)~Inward surface normal \(\hat{n}_i\) and tangent plane define the Coulomb friction cone \(\smash{\|f_{t,i}\| \le \mu\, f_{n,i}}\). (b)~The equivariant force decoder predicts a raw force \(\tilde{f}_i\) (possibly outside the cone). It is decomposed into normal and tangential components and rescaled so the projected output \(\hat{f}_i\) lies inside the cone by construction.}
  \label{fig:contact_frame}
\end{figure}

\textbf{Hand Joint Decoder:}\label{sssec:hand_joint_decoder} Because multiple valid finger configurations exist for the same wrist pose, a deterministic MSE decoder collapses modes to an infeasible mean~\cite{xu2023unidexgrasp}. We instead model \(p_\theta^q(q_h\mid z_O,T_w)\) with a conditional Real-NVP flow~\cite{dinh2017density} using affine coupling layers conditioned on \(\SO(3)\)-invariant features \(\|z_O\|_2\) and the wrist pose.

\textbf{Training Losses:}\label{ssec:trloss} Our training loss is a weighted sum of eight terms
\begin{equation}\label{eq:equidexflowloss}
\mathcal{L} =
  \lambda_{\mathrm{flow}}\,\mathcal{L}_{\mathrm{flow}}
+ \lambda_{q}\,\mathcal{L}_{q}
+ \lambda_{C}\,\mathcal{L}_{C}
+ \lambda_{N}\,\mathcal{L}_{N}
+ \lambda_{F}\,\mathcal{L}_{F}
+ \lambda_{w}\,\mathcal{L}_{w}
+ \lambda_{\mu}\,\mathcal{L}_{\mu}
+ \lambda_{\mathrm{coll}}\,\mathcal{L}_{\mathrm{coll}},
\end{equation}
where the individual components are
\begin{subequations}\label{eq:loss_components}
\small
\setlength{\jot}{1pt}
\begin{alignat}{2}
\mathcal{L}_{\mathrm{flow}} &= \|v_\theta-u_t\|^2
  &&\;\textcolor{Blue}{(\text{flow})} \label{eq:loss_flow} \\
\mathcal{L}_{q} &= -\tfrac{1}{D}\log p_\theta^q(q^*\mid z_O,T_w)
  &&\;\textcolor{Blue}{(\text{joint NLL})} \label{eq:loss_joint} \\
\mathcal{L}_{C} &= \|\hat{C}-C^*\|^2
  &&\;\textcolor{Blue}{(\text{contact})} \label{eq:loss_contact} \\
\mathcal{L}_{N} &= 1-\tfrac{1}{M}\textstyle\sum_i \hat{n}_i^{\top}n_i^{*}
  &&\;\textcolor{Blue}{(\text{normal})} \label{eq:loss_normal} \\
\mathcal{L}_{F} &= \textstyle\sum_i\|\hat{\alpha}_i-\alpha_i^*\|^2
  &&\;\textcolor{Blue}{(\text{force})} \label{eq:loss_force} \\
\mathcal{L}_{w} &= ({}^{O}\hat{r}_w)^\top W_w\,{}^{O}\hat{r}_w
  &&\;\textcolor{Blue}{(\text{wrench})} \label{eq:loss_wrench} \\
\mathcal{L}_{\mu} &= \textstyle\sum_i\bigl[\mathrm{ReLU}(\|\hat{f}_{t,i}\|-\mu\hat{f}_{n,i})+\mathrm{ReLU}(-\hat{f}_{n,i})\bigr]
  &&\;\textcolor{Blue}{(\text{friction})} \label{eq:loss_friction} \\
\mathcal{L}_{\mathrm{coll}} &= \mathcal{L}_{\mathrm{fo}}+\mathcal{L}_{\mathrm{self}}
  &&\;\textcolor{Blue}{(\text{collision})} \label{eq:loss_collision}
\end{alignat}
\end{subequations}
where \(W_w=\mathrm{diag}(I_3,\ell_O^{-2}I_3)\) (with \(\ell_O\) an object length scale that puts force and torque units on the same footing), \(\mu=0.5\), and the object-frame residual wrench is
\begin{equation}
{}^{O}\hat{r}_w=\graspmap(\hat{C})\hat{F}+{}^{O}w_{\mathrm{ext}}=\sum_{i=1}^{M}\!
\begin{bmatrix}{}^{O}\hat{f}_i\\({}^{O}\hat{c}_i-{}^{O}p_{\mathrm{com}})\times{}^{O}\hat{f}_i\end{bmatrix}
+{}^{O}w_{\mathrm{ext}},
\label{eq:obj_wren_res}
\end{equation}
with \(\smash{{}^{O}p_{\mathrm{com}}}\) the object center of mass. In practice \(\mathcal{L}_\mu\) is near-zero because the cone projection is architectural, and we evaluate the contact, normal, and force losses only over fingers with valid ground-truth contacts. The wrench and friction terms stay coordinate-invariant in a common frame, but they vary under physical object rotations because \(\smash{{}^{O}g=({}^{W}R_O)^\top{}^{W}g}\) makes the external wrench orientation-dependent. Gravity therefore enters as physical conditioning, not as a symmetry to remove, and the per-loss weights appear in \Cref{tab:hyperparams}.

\begin{table}[htb]
\caption{\scshape Weighted Loss-Component Magnitudes (Allegro \textsc{Full} Model, Validation Mean over the Final 50 Logged Evaluations of the 16.2K-Step Run).}
\label{tab:loss_components}
\centering
\small
\begin{tabular}{@{}lccclccc@{}}
\toprule
Term & \(\lambda_i\) & Raw \(\mathcal{L}_i\) & Weighted \(\lambda_i \mathcal{L}_i\) & Term & \(\lambda_i\) & Raw \(\mathcal{L}_i\) & Weighted \(\lambda_i \mathcal{L}_i\) \\
\midrule
Flow          & 1.0   & 0.916  & 0.916 & Force         & 1.0   & 4.499 & 4.499 \\
Hand-q (NLL)  & 1.0   & 0.185  & 0.185 & Wrench        & 0.1   & 1.181 & 0.118 \\
Contact       & 100.0 & 0.0028 & 0.280 & Friction      & 0.1   & 0.000 & 0.000 \\
Normal        & 1.0   & 1.068  & 1.068 & Collision     & 100.0 & 0.0043 & 0.430 \\
\bottomrule
\end{tabular}
\end{table}

Before establishing end-to-end equivariance, we isolate the force pathway and its implication for the overall learning architecture.

\begin{proposition}[Force Equivariance]\label{prop:force_eq}
Let \(\alpha_i\in\R^3\) be local force coefficients predicted in the contact frame \(B_i\) and let \(f_i=B_i\alpha_i\) be the reconstructed Cartesian force. If \(\alpha_i\) is computed as an invariant function of \(\mathcal{P}\) (i.e.\ \(\alpha_i(A\mathcal{P})=\alpha_i(\mathcal{P})\) for every \(A\in\SE(3)\)) and \(B_i\) transforms equivariantly as \(B_i'=R_AB_i\), then \(f_i\) is equivariant: \(f_i(A\mathcal{P})=R_Af_i(\mathcal{P})\).
\end{proposition}

\begin{proof}
By the composition of invariant and equivariant maps~\cite{deng2021vn},
\(f_i(A\mathcal{P})=B_i'\,\alpha_i(A\mathcal{P})
=(R_AB_i)\,\alpha_i(\mathcal{P})
=R_A\bigl(B_i\alpha_i(\mathcal{P})\bigr)
=R_Af_i(\mathcal{P})\).
\end{proof}

\begin{remark}[Force Equivariance Implication]
A network that predicts Cartesian forces directly can encode spurious world-frame correlations (e.g., a learned bias toward the world \(z\)-axis). Predicting in local coordinates and reconstructing the Cartesian force geometrically gives equivariance for free and is essential for the end-to-end guarantee in \Cref{thm:equiv}. Gradients flow from the force decoder through the predicted contacts to the contact decoder, coupling the wrench-balance objective with contact placement.
\end{remark}
\subsection{Wrist Refinement}\label{ssec:wristproj}
\begin{wrapfigure}[12]{r}{0.42\columnwidth}
\centering\vspace{-20pt}
\includegraphics[width=0.40\columnwidth]{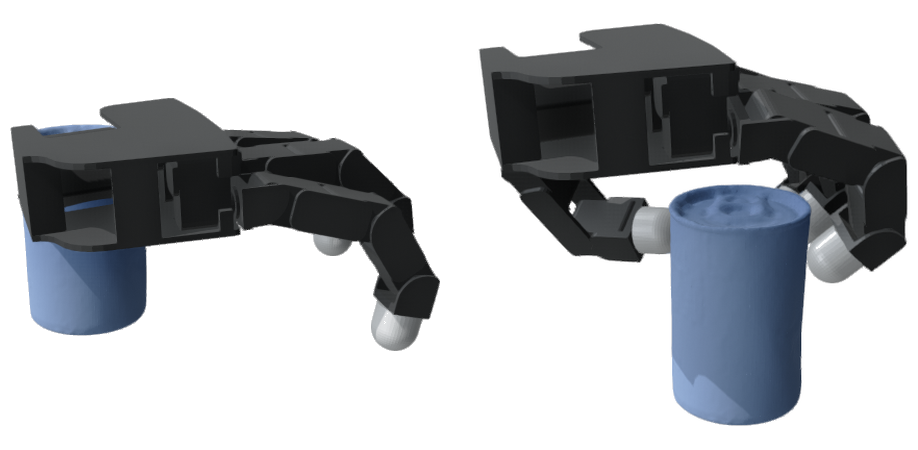}
\caption{\textbf{Wrist Refinement.} Pre-contact IK (\emph{left}): the hand at the raw decoded wrist pose and joint configuration \(\hat{q}_h\) hovers above the YCB tomato soup can, and the predicted contacts \(\hat{C}\) are not realized by the kinematic
fingertips \(g_i(T_w, q_h)\). Contact IK (\emph{right}): minimizing \eqref{eq:tto_obj} seats each fingertip on the surface, driving \(g_i(T_w, q_h) \to \hat{C}_i\).}
\label{fig:tto_prepost}
\vspace{-6pt}
\end{wrapfigure}
The flow-decoded wrist typically places the fingertips a few centimeters off the object surface. We close it with a \emph{test-time optimization} (TTO) step that jointly refines the full \(6\)-DoF wrist pose and the \(D\)-DoF joint vector so the kinematic fingertips reach the predicted contacts (\Cref{fig:tto_prepost}). Let \(g_i(T_w,q_h)\) be the \(i\)-th fingertip position under the differentiable forward kinematics of the hand and let \(\hat{C}_i\) be the decoded target contacts. Starting from the decoded pose \((\hat{T}_w,\hat{q}_h)\), we minimize the loss function
\begin{align}\label{eq:tto_obj}
\mathcal{L}_{\mathrm{TTO}}=&\sum_{i}\bigl\|g_i(T_w,q_h)-\hat{C}_i\bigr\|^2
+w_{\mathrm{tr}}\,\rho\bigl((T_w,q_h),(\hat{T}_w,\hat{q}_h)\bigr)\nonumber\\
&+w_{\mathrm{pen}}\sum_i\mathrm{ReLU}\bigl(-\Phi_M(g_i)\bigr)^2
+\mathcal{B}(q_h),
\end{align}
where \(\Phi_M:\R^3\to\R\) is the signed distance to the object mesh (positive outside), \(\rho\) is a trust-region penalty on deviation from the decoded pose, the third term penalizes fingertips inside the mesh, and \(\mathcal{B}\) is a soft joint-limit barrier. We optimize \eqref{eq:tto_obj} by gradient descent (\(200\) steps, learning rate \(0.02\), \(w_{\mathrm{tr}}=0.05\), \(w_{\mathrm{pen}}=50\), all repeated in \Cref{tab:hyperparams}), parameterizing the wrist update in the tangent space of \(\SE(3)\) so iterates stay on the manifold. An optional finger-finger self-collision term is disabled (\(w_{\mathrm{self}}=0\)) because it did not improve seated contact. Every term in \eqref{eq:tto_obj} is frame-invariant, so the refinement map \(\phi_{\mathrm{ref}}\) commutes with any \(A\in\SE(3)\) (\Cref{apx:eqv_prf}).

\subsection{Equivariance Guarantee}\label{ssec:eq_guarantee}
The cone projection is single-frame: it enforces friction compliance in the same contact frame the local coefficients live in. A different normal at evaluation time can therefore invalidate the cone bound, as the following proposition makes precise.

\begin{proposition}[Normal Consistency]\label{prop:norm_cons}
Let \(\hat{n}_i\) be the unit normal used to construct the contact frame \(B_i=[t_{i,1},t_{i,2},\hat{n}_i]\) in which the force decoder predicts local coefficients \(\alpha_i=(f_{t1},f_{t2},f_n)^\top\) and the cone projection enforces \(\sqrt{f_{t1}^2+f_{t2}^2}\leq\mu f_n\), \(f_n\geq 0\). Let \(\tilde{n}_i\) be a different unit normal used to evaluate friction compliance at test time, with frame \(\tilde{B}_i=[s_{i,1},s_{i,2},\tilde{n}_i]\). Decompose the reconstructed force \(f_i=B_i\alpha_i\) in the evaluation frame as \(\tilde{\alpha}_i=\tilde{B}_i^\top f_i\). Then the evaluation-frame normal component is
\begin{equation}
\tilde{f}_{n,i}=\hat{n}_i^\top\tilde{n}_i\,f_{n,i}+(\text{tangential cross-terms}),
\label{eq:normal_cross}
\end{equation}
and the friction-cone condition in the evaluation frame is violated whenever \(\|\tilde{f}_{t,i}\|>\mu\,\tilde{f}_{n,i}\), which can occur even when \(\alpha_i\) satisfies the generation-frame cone exactly, provided \(\hat{n}_i\neq\tilde{n}_i\).
\end{proposition}

\begin{proof}
Write \(f_i = B_i\alpha_i = f_{t1}\,t_{i,1}+f_{t2}\,t_{i,2}+f_n\,\hat{n}_i\). Project onto \(\tilde{n}_i\):
\[
\tilde{f}_{n,i} = f_i^\top\tilde{n}_i
= f_{t1}(t_{i,1}^\top\tilde{n}_i)+f_{t2}(t_{i,2}^\top\tilde{n}_i)+f_n(\hat{n}_i^\top\tilde{n}_i).
\]
When \(\hat{n}_i=\tilde{n}_i\), the tangent terms vanish (\(t_{i,k}\perp\hat{n}_i\)) and \(\tilde{f}_{n,i}=f_n\), preserving the cone. When \(\hat{n}_i\neq\tilde{n}_i\), the tangent cross-terms are generically nonzero, the normal projection shrinks by \(\cos\angle(\hat{n}_i,\tilde{n}_i)<1\), and tangential leakage into the evaluated normal direction breaks the cone bound.
\end{proof}

\begin{remark}[Normal Consistency Implication] An architectural cone projection enforces friction compliance only relative to the contact frame it was constructed with: if a different normal estimate (e.g., centroid-based vs.\ mesh-derived) is used downstream for scoring or simulation, the compliance guarantee does not transfer and the measured violation rate can approach 100\% even for a perfectly trained model. More generally, any equivariant and consistently applied frame convention suffices. What breaks the invariant is using one convention at generation time and a different one at evaluation time. Our architecture uses a single geometry-derived normal source throughout training, generation, and evaluation.
\end{remark}

\begin{theorem}[End-to-End SE(3) Architectural Equivariance]\label{thm:equiv}
Let \(\mathcal{P}\) be an object point cloud and let \(A=(R_A,x_A)\in\SE(3)\). The \equidexflow{} inference pipeline (VN-DGCNN encoder, SE(3) flow backbone, contact / normal / force heads with cone projection, and wrist refinement \(\phi_{\mathrm{ref}}\)) produces an output \(\G=(T_w,q_h,C,N,F)\) satisfying
\begin{equation}
\G\sim p_\theta(\cdot\mid\mathcal{P})\;\Longrightarrow\;A\cdot \G\sim p_\theta(\cdot\mid A\mathcal{P}),
\end{equation}
i.e., the pipeline is SE(3)-equivariant in the sense of \Cref{def:eq_generator}.
\end{theorem}

\begin{proof}[Proof Sketch]
The VN-DGCNN encoder is \(\SO(3)\)-equivariant by construction~\cite{deng2021vn}, so \(z_O(A\mathcal{P})=z_O(\mathcal{P})R_A^\top\). Because every downstream module (flow backbone, contact/normal decoders, force decoder with cone projection, and wrist refinement) is either a VN layer conditioned on equivariant inputs or operates in a local frame that co-rotates with \(A\), each stage individually commutes with the group action. The claim follows by composing these equivariant maps. The full stage-by-stage argument appears in \Cref{apx:eqv_prf}, and we provide a discussion on joint configuration invariance in \Cref{apx:remark_jnt_equv}.
\end{proof}

\subsection{Inference \& Candidate Ranking}\label{ssec:inference}
At inference time, we draw \(K = 10\) candidates and rank them by the composite score \(J\) of \eqref{eq:ranking}, instantiated with \(Q_{\mathrm{geom}} = -\mathcal{L}_{\mathrm{coll}}\), \(Q_{\mathrm{phys}} = -\|\graspmap(\hat{C})\hat{F} + w_{\mathrm{ext}}\|_2 - \operatorname{FVR}(\hat{C},\hat{F})\), and \(Q_{\mathrm{task}} = \sigma_{\min}(\graspmap(\hat{C}))\), where \(\graspmap\in\R^{6\times 3M}\) is the grasp map relating per-contact Cartesian forces to the resultant object-frame wrench and \(\sigma_{\min}(\cdot)\) is the minimum singular value. The risk term is \(Q_{\mathrm{risk}} = \tfrac{1}{M}\sum_{i=1}^{M}\operatorname{Var}(\hat{c}_i)\), and we set the weights to \((\beta_1,\beta_2,\beta_3,\beta_4) = (1, 2, 1, 0.5)\). We then apply the wrist refinement of \Cref{ssec:wristproj} to the top candidate, and \Cref{thm:equiv} certifies that the full pipeline preserves SE(3) equivariance.

\section{Experiments}\label{sec:experiments}
We describe the training corpus, the training protocol, and the ablation baselines before turning to results.
\subsection{Grasp Dataset Generation}
We synthesize training data with FRoGGeR~\cite{li2023frogger}, a Drake-based~\cite{drake} min-weight-metric optimizer that solves a nonlinear program over the hand's forward-kinematic tree, maximizing worst-case wrench resistance in the grasp wrench space subject to surface contact, joint-limit, self-collision, and object-collision constraints (\Cref{apx:graspsynth}). The planner reasons over the full kinematic chain of the 16-DoF Allegro Hand (four fingers, each a
four-revolute serial chain), placing fingertip contact patches on the object surface via inverse kinematics and computing the grasp map from the resulting contact normals and positions. We then certify force closure post hoc by checking that the origin lies strictly inside the convex hull of primitive wrenches generated under Coulomb friction cones at every contact.

We generate 100 candidate grasps per object across 81 rigid bodies (49~EGAD~\cite{morrisonEGADEvolvedGrasping2020b}, 28~YCB~\cite{calli2015ycb}, 4~primitives), retaining those that pass force-closure certification for 8{,}100 grasps total, split 80\,/\,10\,/\,10 into 6{,}480 train, 809 validation, and 811 test. Each record stores the object point cloud (\(512\) points), the wrist pose \(T_w\in\SE(3)\), joint angles \(q_h\in\R^{16}\), per-fingertip contacts \(C_i\) and outward surface normals \(n_i\), and LP-optimal contact forces \(f_i\) decomposed into the contact-frame basis \(B_i\). We apply random \(\SO(3)\) augmentation to the point cloud and grasp during training, which the equivariant encoder processes without loss of generality. The full record schema and object taxonomy appear in \Cref{apx:dataprov}.

\subsection{Training \& Evaluation Metrics}\label{ssec:traineval}
We train with Adam~\cite{kingma2015adam} on the schedule in \Cref{tab:hyperparams}. Validation curves (\Cref{fig:train_curves}) show collision and friction hinges staying inactive, consistent with cone projection removing violations by construction (\Cref{tab:loss_components}). We draw 10 candidates per object, rank by \(J\), and report three metric families: \emph{contact fidelity} (per-finger mean error in meters, fraction within \(1\,\mathrm{cm}\)), \emph{force fidelity} (magnitude error in N, direction error in degrees, friction violation rate (FVR)), and \emph{physics score} (Top-1 and mean Top-3 composite \(J\) over \(K{=}10\), plus the wrench residual \(\|\graspmap\hat{F}+w_{\mathrm{ext}}\|_2\)).

\subsection{Baselines}\label{ssec:baselns}%
Since no published method jointly generates a \((q,C,F)\) distribution from point clouds, we ablate by progressive decoder addition over a shared encoder, flow backbone, and training data, varying only the active heads and losses. \textsc{PoseOnly} uses the flow backbone and joint decoder (wrist and joints only). \textsc{ContactOnly} adds the contact decoder and loss. \textsc{GeomOnly} adds the force decoder but zeros the physics losses (\(\lambda_w{=}\lambda_\mu{=}\lambda_{\mathrm{coll}}{=}0\)), training forces from labels alone. \textsc{Full}, reported as \equidexflow{} in \Cref{tab:results}, activates all heads and physics losses.

\section{Results \& Discussion}\label{sec:resdis}
\begin{figure*}[t]
\centering
\begin{subfigure}[t]{0.60\textwidth}
\centering
\includegraphics[width=.98\textwidth]{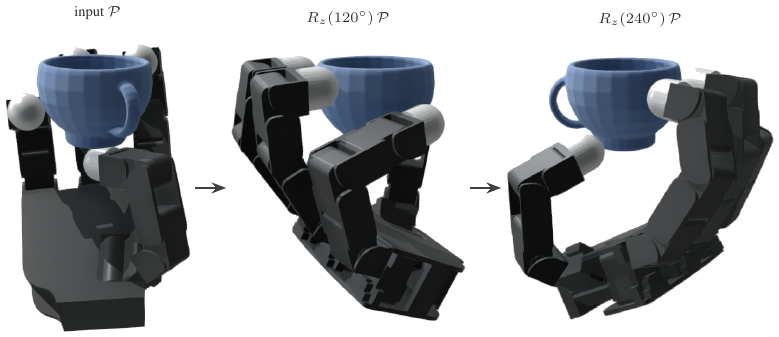}
\caption{Equivariant Co-Rotation}
\label{fig:gallery}
\end{subfigure}
\hfill
\begin{subfigure}[t]{0.38\textwidth}
\centering
\includegraphics[width=\textwidth]{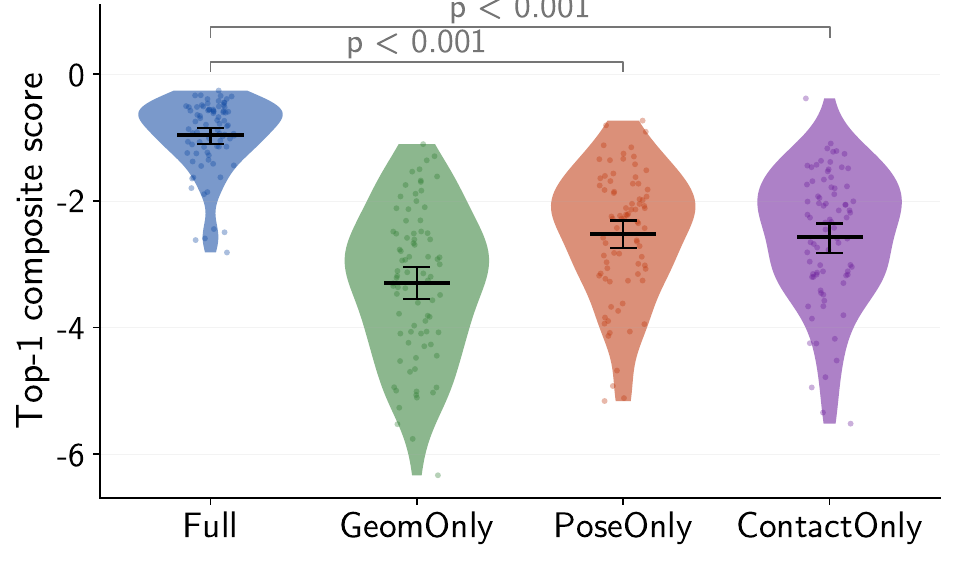}
\caption{Top-1 Score Distribution}
\label{fig:score_dist}
\end{subfigure}
\caption{\textbf{Equivariance and Grasp Quality.} \emph{(a)} Grasps co-rotate with the input under \(R_z\) rotation (wrist residual \({<}0.04^\circ\), \(\max\Delta q_h{=}0\), \Cref{tab:equivariance_binned}). \emph{(b)} Top-1 composite score over 81 test objects (means, BCa bootstrap 95\% intervals, Holm-corrected Wilcoxon \(p\)).}
\label{fig:gallery_score}
\end{figure*}
\Cref{tab:results} summarizes quantitative results. All four variants achieve 0\% friction violations, confirming the architectural guarantee: cone projection enforces Coulomb feasibility regardless of which losses are active. Disabling the projection at test time removes the guarantee (\Cref{ssec:inference_ablation}). Because all variants share the same 81-object split, the comparison isolates each decoder head and loss term. Contact error sits in a tight 0.040 to 0.042\,m band across variants, a gap the wrist refinement of \Cref{ssec:wristproj} closes at execution.

\begin{table}[htb]
\caption{\scshape Grasp Quality on the Test Set (811 Grasps, 10 Samples Each). Friction violation is the fraction with \(\|f_t\|>\mu f_n\), and wrench residual is \(\|\graspmap\hat{F}+w_{\mathrm{ext}}\|_2\). The shaded row is the full model.}
\label{tab:results}
\centering
\small
\setlength{\tabcolsep}{2pt}
\renewcommand{\arraystretch}{0.8}
\begin{tabular}{lcccccc}
\toprule
Method
  & \makecell{Contact\\Err (m)\(\downarrow\)}
  & \makecell{Force\\Err (N)\(\downarrow\)}
  & \makecell{Fric.\ Viol.\\(\%)\(\downarrow\)}
  & \makecell{Top-1\\Score\(\uparrow\)}
  & \makecell{Top-3\\Score\(\uparrow\)}
  & \makecell{Wrench\\Res (Nm)\(\downarrow\)} \\
\midrule
\textsc{PoseOnly}    & \textbf{0.040} & \textbf{1.57} & \textbf{0.0} & \(-\)2.52 & \(-\)3.25 & 1.29 \\
\textsc{ContactOnly} & 0.041 & 1.84 & \textbf{0.0} & \(-\)2.57 & \(-\)3.46 & 1.36 \\
\textsc{GeomOnly}    & 0.041 & 1.84 & \textbf{0.0} & \(-\)3.29 & \(-\)4.02 & 1.58 \\
\rowcolor{cRowOurs}
\textbf{\equidexflow~(ours)} & 0.042 & 1.99 & \textbf{0.0} & \textbf{\(-\)0.96} & \textbf{\(-\)1.18} & \textbf{0.46} \\
\bottomrule
\end{tabular}
\end{table}

\subsection{Ablation Results}\label{ssec:ablation_results}
The \textsc{Full} model attains the best composite scores (Top-1 \(-0.96\), Top-3 \(-1.18\), \Cref{fig:score_dist}) and the lowest wrench residual (\(0.46\,\mathrm{Nm}\)). \textsc{GeomOnly}, which zeros the physics loss weights, trails by \(2.33\) on Top-1 and \(2.84\) on Top-3, the largest gap among comparisons, so the physics losses rather than architecture drive the separation. A Holm-corrected Wilcoxon signed-rank test over \(n{=}81\) paired objects, with BCa bootstrap 95\% confidence intervals over \(R{=}10{,}000\) resamples, confirms \textsc{Full} is better than every ablation (\(p<0.001\)). The largest effect is against \textsc{GeomOnly} (Cohen's \(d_z{=}2.50\) on Top-1 and \(3.70\) on Top-3, far past the conventional \(|d_z|>0.8\) threshold for a large effect). \textsc{PoseOnly} (\(-2.52\)) and \textsc{ContactOnly} (\(-2.57\)) are comparable, while \textsc{GeomOnly} (\(-3.29\)) is weakest because training the force head without the physics losses that consume it yields miscalibrated forces the score penalizes. Force direction error sits at \(54\) to \(65^\circ\) across all variants because the dataset normals are centroid-derived, so depth-sensor normals at inference would lift this shared ceiling. The wrench residual is direction-dominated, consistent with \Cref{prop:norm_cons}.

\subsection{Inference-Time Ablation}\label{ssec:inference_ablation}
We validate our architectural claims by disabling components at inference time on the \textsc{Full} checkpoint without retraining (\Cref{tab:inference_ablations}).

\begin{table}[t]
\captionof{table}{\scshape Inference-Time Ablations (\textsc{Full} Checkpoint, 81 Test Objects, 10 Candidates Each). The shaded row disables the architectural cone projection.}
\label{tab:inference_ablations}
\centering\small
\begin{tabular}{@{}lrrr@{}}
\toprule
Configuration & Top-1\(\uparrow\) & Top-3\(\uparrow\) & FVR\(\downarrow\) \\
\midrule
Stochastic        & \(-\)0.95 & \(-\)1.22 & 0.0 \\
Deterministic\ (\(z{=}0\))   & \(-\)0.97 & \(-\)1.24 & 0.0 \\
\rowcolor{cRowOff}
No cone projection    & \(-\)5.92 & \(-\)5.95 & 100 \\
\bottomrule
\end{tabular}
\end{table}

\textbf{Cone Projection:} Removing the friction-cone projection at test time degrades the composite score by over \(6\times\) (Top-1 from \(-0.96\) to \(-5.92\)) and drives the friction violation rate to 100\%. The raw force vectors from the VN layers carry no inductive bias toward compressive, friction-feasible forces, and the cone projection, a softplus on the normal component followed by tangential clipping at \(\mu f_n\), is what converts arbitrary equivariant vectors into physically valid contact forces. The 0\% violation rate in \Cref{tab:results} is therefore a structural guarantee of the architecture, not an artifact of the training-data distribution.

\textbf{Stochastic vs.\ Deterministic Sampling:} The conditional flow decoder's stochastic mode (\(z \sim \mathcal{N}(0,I)\)) and deterministic mode (\(z=0\), the learned distribution mode) produce indistinguishable quality (Top-1: \(-0.95\) vs.\ \(-0.97\), Top-3: \(-1.22\) vs.\ \(-1.24\)). The flow learns a high-quality mode, and stochastic sampling provides diversity without degradation. Deployments that require deterministic output can use the mode directly with no loss in grasp quality.

\subsection{Empirical Equivariance}\label{ssec:empirical_equivariance}
We verify \Cref{thm:equiv} by sampling 200 \(\SO(3)\) rotations, binning by angle (\(30^\circ\) increments), and measuring residuals against the identity-input prediction (\Cref{tab:equivariance_binned}). Wrist rotation error stays below \(0.04^\circ\) and translation below \(2\times 10^{-3}\,\mathrm{mm}\) across all bins with no degradation at large angles. Per-joint deviation is identically zero across all 2000 samples because the Real-NVP decoder depends only on \(\SO(3)\)-invariant features (\Cref{apx:remark_jnt_equv}). The wrist residual stems from FP32 accumulation in the RK4 integrator. \Cref{fig:eq:six_test} shows six test objects at three input rotations. Per-axis residuals confirm no directional bias, and the FP32 wrist residual vanishes under FP64 arithmetic. The wrist refinement of \Cref{ssec:wristproj} re-solves contact IK against the rotated surface and therefore introduces a small orientation-dependent joint perturbation. We apply this refinement identically at every orientation as a post-hoc execution step, so it is a feature of execution and not a property of the generative model.
\begin{figure}[htb]
\centering
\begin{minipage}[c]{.32\columnwidth}
  \centering
  \includegraphics[width=\linewidth]{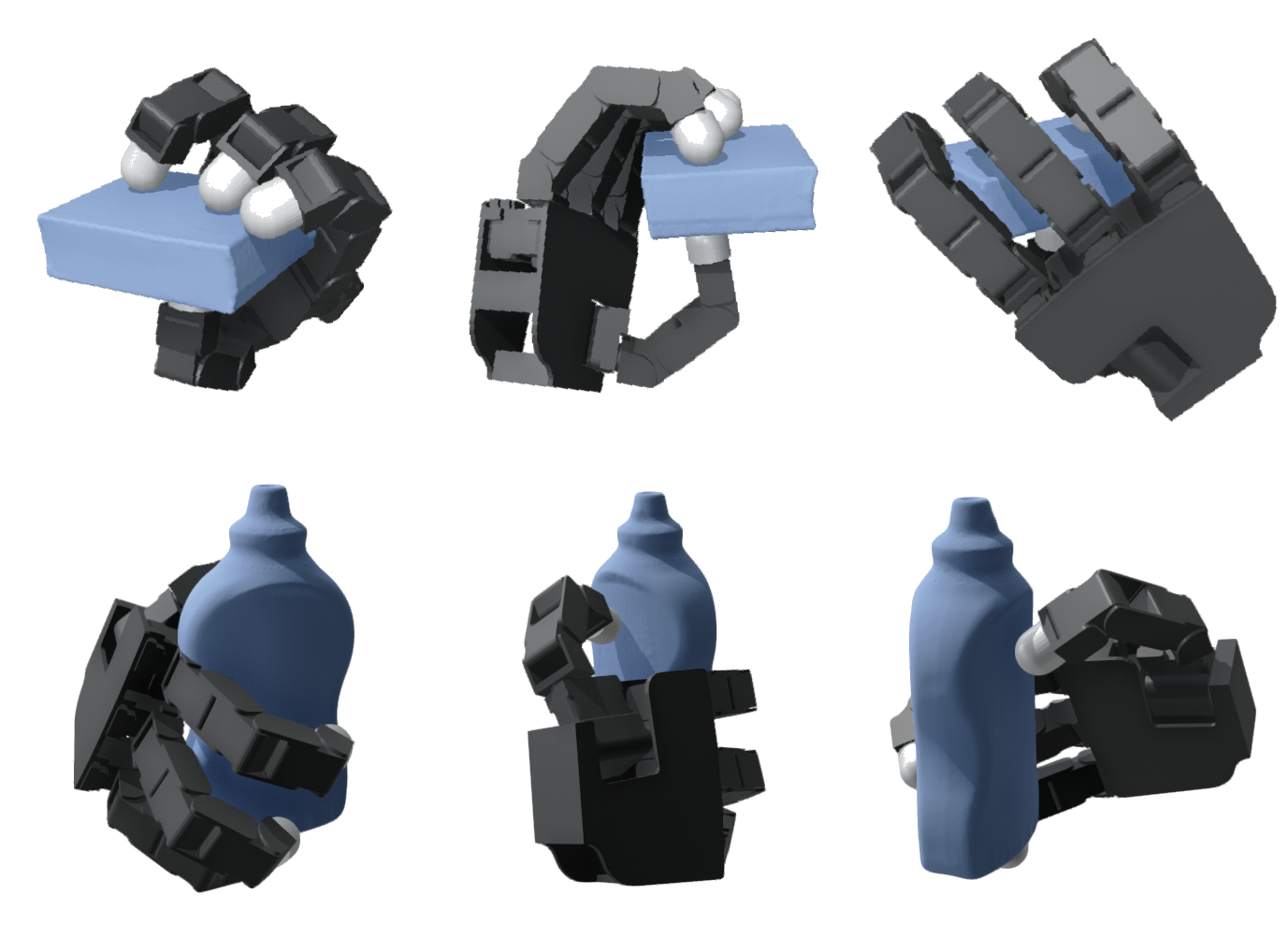}
\end{minipage}%
\hfill
\begin{minipage}[c]{.32\columnwidth}
  \centering
  \includegraphics[width=\linewidth]{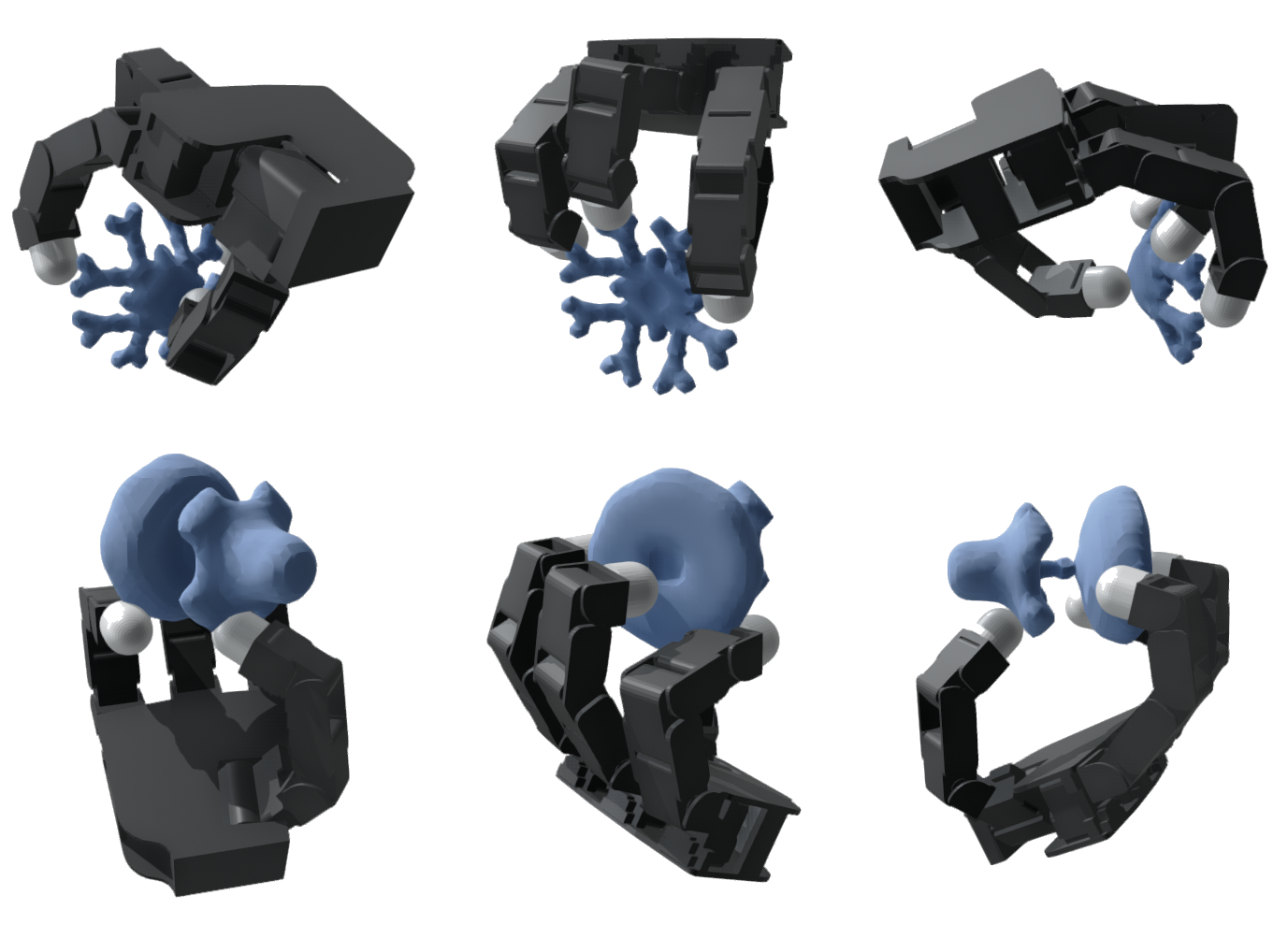}
\end{minipage}%
\hfill
\begin{minipage}[c]{.32\columnwidth}
  \centering
  \includegraphics[width=\linewidth]{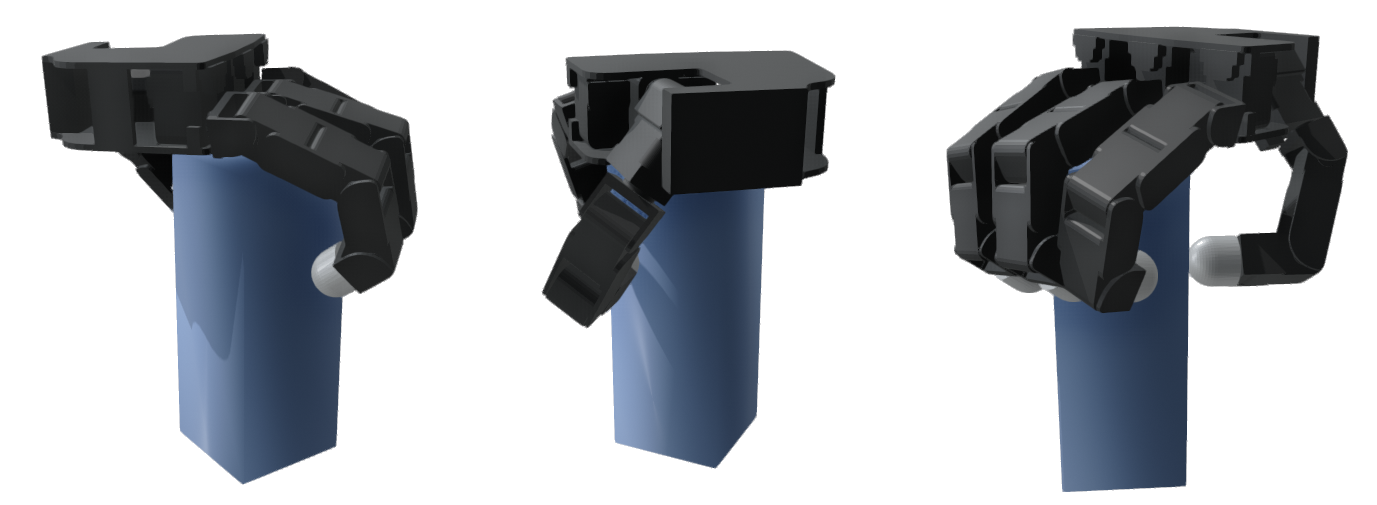}\\[4pt]
  \includegraphics[width=\linewidth]{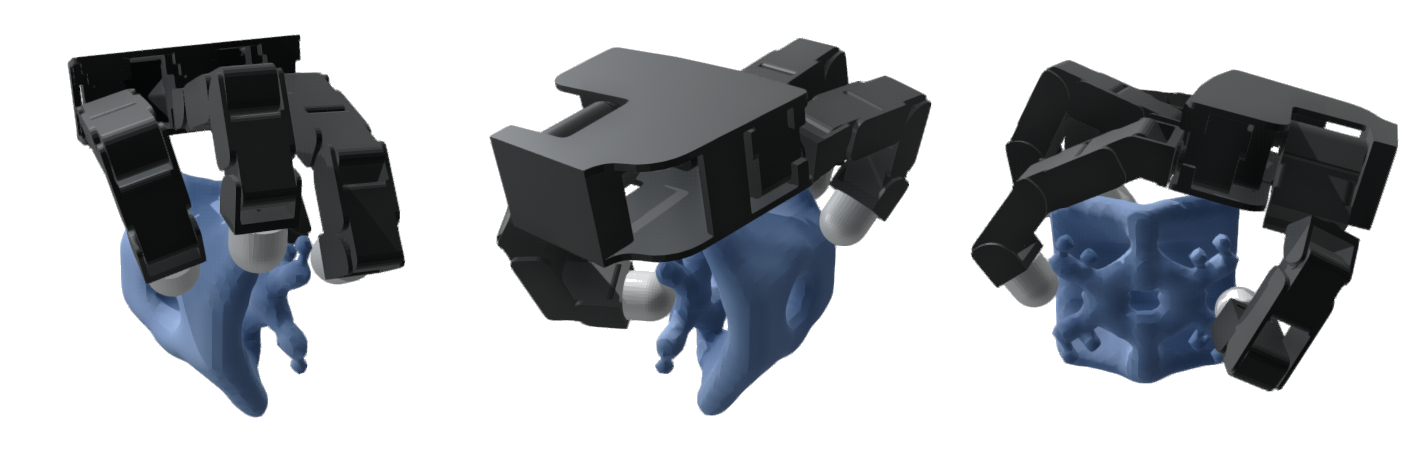}
\end{minipage}
\caption{\textbf{End-to-End SE(3) Equivariance of \equidexflow{} on the 16-DoF Allegro Hand.} Six objects (left to right, then top to bottom in triples: the YCB pudding box, EGAD!\ G6, the box primitive, the YCB mustard bottle, EGAD! E4, and EGAD! G5) across three input rotations (columns: \(0^\circ\), \(120^\circ\), \(240^\circ\) about the vertical axis). The wrist pose and finger configuration co-rotate with the object. The wrist residuals stay below \(0.04^\circ\) (\Cref{tab:equivariance_binned}) and the raw-flow joint deviation is identically zero. The grasps shown are also seated on the surface by the wrist refinement of \Cref{ssec:wristproj}.}
\label{fig:eq:six_test}
\end{figure}

\subsection{Diversity and Coverage}\label{ssec:diversity}
We generate \(K=20\) candidates per object per variant and measure sample diversity and coverage (\Cref{tab:diversity}). Rotation and joint-angle diversity are comparable across variants (\(\sim\)41\(^\circ\) and \(\sim\)0.32\,rad), reflecting the shared SE(3) flow backbone and conditional joint flow. Wrist-translation diversity is wider for \textsc{Full} and \textsc{GeomOnly} (\(\sim\)60\,mm), which use the full generative architecture, than for the head-ablated \textsc{PoseOnly} and \textsc{ContactOnly} variants (\(\sim\)32\,mm). Coverage@\(k\) (C@\(k\) for brevity) is the fraction of test objects where at least one of \(k\) candidates clears the quality threshold (score \(>-4.0\)). It is highest for \textsc{Full} (100\%), followed by \textsc{ContactOnly} (99\%) and \textsc{PoseOnly} (98\%), with \textsc{GeomOnly} weakest at 88\%. Coverage@\(k\) is flat across \(k=1,4,8,20\), so the residual misses are systematic rather than stochastic. \textsc{Full} also shows the tightest contact spread (55.9\,mm vs.\ \(\sim\)130\,mm), consistent with its physics losses concentrating candidates on wrench-feasible contact placements.

\begin{table}[ht!]
\centering
\begin{minipage}[t]{0.46\columnwidth}
\centering
\captionof{table}{\scshape Equivariance Error by Rotation Angle. The \(\max\Delta q_h\) column is identically zero by \Cref{apx:remark_jnt_equv}.}
\label{tab:equivariance_binned}
\small
\setlength{\tabcolsep}{4pt}
\begin{tabular}{@{}lrrr@{}}
\toprule
Bin & \(\Delta R_w\)(\(^\circ\)) & \(\Delta x_w\)(mm) & \(\max\Delta q_h\)(\(^\circ\)) \\
\midrule
  \(30\)--\(60\)  & \(3.7\!\times\!10^{-2}\) & \(2.0\!\times\!10^{-3}\) & \(0.00\) \\
  \(60\)--\(90\)  & \(3.7\!\times\!10^{-2}\) & \(2.0\!\times\!10^{-3}\) & \(0.00\) \\
  \(90\)--\(120\) & \(3.6\!\times\!10^{-2}\) & \(1.0\!\times\!10^{-3}\) & \(0.00\) \\
  \(120\)--\(150\) & \(3.6\!\times\!10^{-2}\) & \(1.7\!\times\!10^{-3}\) & \(0.00\) \\
  \(150\)--\(180\) & \(3.6\!\times\!10^{-2}\) & \(2.2\!\times\!10^{-3}\) & \(0.00\) \\
\midrule
\multicolumn{4}{@{}l}{\textit{Per-axis (50\(\times\)5, wrist only)\(^\dagger\):}} \\
  \(X\) & \(3.5\!\times\!10^{-2}\) & \(3.6\!\times\!10^{-5}\) & -- \\
  \(Y\) & \(4.2\!\times\!10^{-2}\) & \(3.5\!\times\!10^{-5}\) & -- \\
  \(Z\) & \(3.5\!\times\!10^{-2}\) & \(3.6\!\times\!10^{-5}\) & -- \\
\bottomrule
\multicolumn{4}{@{}l}{\tiny\(^\dagger\)Wrist-only directional-bias check.}
\end{tabular}
\end{minipage}\hfill
\begin{minipage}[t]{0.50\columnwidth}
\centering
\captionof{table}{\scshape Diversity and Coverage (\(K{=}20\) samples/object, 81 objects).}
\label{tab:diversity}
\small
\setlength{\tabcolsep}{2pt}
\begin{tabular}{@{}lrrrrrr@{}}
\toprule
Variant & Tr. & Rot. & Jt. & Spr. & C@1 & C@8 \\
        & (mm) & (\(^\circ\)) & (rad) & (mm) & & \\
\midrule
\rowcolor{cRowOurs}
\textsc{Full}        & 59.8 & 41.1 & 0.321 & 55.9  & 100\% & 100\% \\
\textsc{GeomOnly}    & 63.0 & 40.5 & 0.323 & 132.5 & 88\%  & 88\% \\
\textsc{PoseOnly}    & 31.8 & 41.1 & 0.323 & 136.0 & 98\%  & 98\% \\
\textsc{ContactOnly} & 31.6 & 41.5 & 0.322 & 135.9 & 99\%  & 99\% \\
\bottomrule
\end{tabular}
\end{minipage}
\end{table}

\begin{figure}[htb]
\centering
\begin{tabular}{@{}c@{\hskip 4pt}c@{\hskip 4pt}c@{}}
\includegraphics[width=0.325\linewidth]{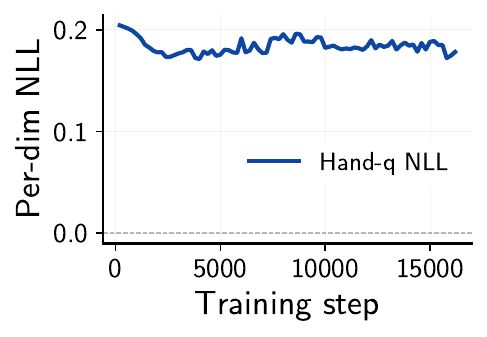} &
\includegraphics[width=0.325\linewidth]{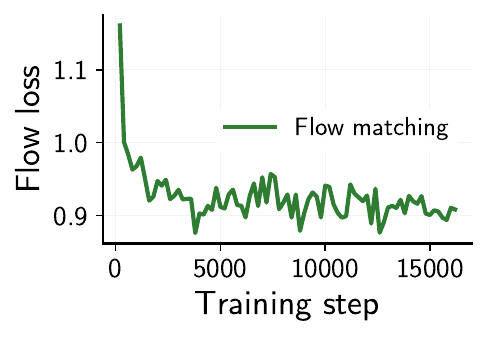} &
\includegraphics[width=0.325\linewidth]{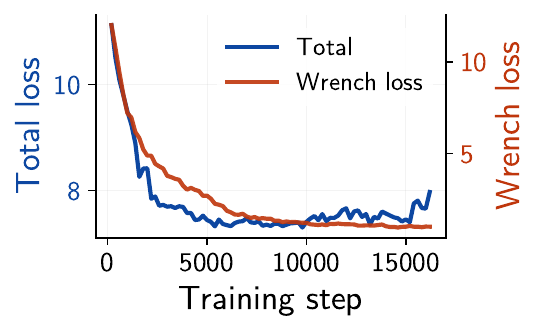}
\end{tabular}
\caption{\textbf{Validation Curves (\textsc{Full} Model).} Joint-angle NLL (\emph{left}, \(\to\!0.18\) nats), SE(3) flow loss (\emph{center}, \(\to\!0.9\)), and total loss (\emph{right}, \(\sim\!11\!\to\!7.5\)) with the wrench-balance loss (\(\to\!1\)). Loss components in \Cref{tab:loss_components}.}
\label{fig:train_curves}
\end{figure}

\subsection{Hardware Experiments}\label{ssec:hardware}
\equidexflow{} targets the Allegro Hand. However, our hardware platform comprises a 6-DoF ZArm manipulator\footnote{The ZArm 6-DoF arm: \url{https://www.frtech.fr/FR/5.html}.} fitted with the LEAP Hand~\cite{shaw_leap_2023}, so we retarget the decoded fingertip contacts through LEAP inverse kinematics (Drake's \texttt{InverseKinematics}) rather than re-training. Across 500 grasps over five objects the retarget converges for 499, with a mean fingertip residual of \(14\,\mathrm{mm}\), roughly \(3\times\) tighter than the \SIrange{29}{54}{\milli\meter} residual of a naive joint-offset transfer that ignores the finger-length difference between hands. Equivariance (\Cref{thm:equiv}) fixes \emph{which} pose to execute at each object rotation (the canonical grasp transformed by the same group element). Reachability is a separate question set by the arm's fixed, non-rotation-symmetric workspace. We select the highest-ranked grasp whose full pick sequence (pre-grasp, grasp, closure, lift) clears the worktable (\Cref{fig:hardware}). For the four representative objects in \Cref{tab:hw_reach}, the retargeted top-ranked grasp is reachable at both the canonical pose and a \(120^\circ\) rotation about the vertical axis. On the physical robot, all four asymmetric objects (the cube, box primitive, potted meat can, and mustard bottle) complete the full open-loop pick and hold at both the canonical pose and its \(120^\circ\) co-transform, and the symmetric cylinder primitive and tennis ball complete it at the canonical pose (\Cref{sec:supp_hw_protocol}). These outcomes are the 16-DoF dexterous counterpart to prior parallel-jaw equivariant grasping~\cite{lim2025equigraspflow}. Videos on the project website show the simulation validation alongside the canonical and \(120^\circ\) hardware executions (pre-grasp through lift).

\paragraph{Hardware-Feasible Refinement:}
For hardware execution, naive grasp retargeting is hardly sufficient. For instance, the retargeted hand vector \(\tilde{q}_h\) may often seat one or more joints against URDF hard limits (we recorded the most cases with the LEAP Hand's \texttt{thumb\_cmc\_rotation}), and the current-limited Dynamixel Mode-5 actuators cannot drive against that limit. This causes the affected fingertip to stall short of its target \(p_i\), and the four-finger wrap degenerates. We restore a comfort margin by projecting \(\tilde{q}_h\) one finger at a time. For the LEAP hand each of the \(M{=}4\) fingers contributes four joints to the \(D{=}16\) hand vector, so the sub-vector for finger \(m\) is \(\smash{q_h^{(m)} \in \R^{4}}\). For each \(m\) we solve
\begin{equation}\label{eq:hw_refine}
\min_{q_h^{(m)}}\;\sum_{j=1}^{4}\!\left(\frac{q_{h,j}^{(m)} - q_{\mathrm{mid},j}}{q_{\mathrm{half},j}}\right)^{\!\!4}
+ w_{\mathrm{tip}}\,\bigl\|g_m(T_w,q_h^{(m)}) - g_m(T_w,\tilde{q}_h^{(m)})\bigr\|^2
\;\;\text{s.t.}\;\; q_{\mathrm{lo}} + \varepsilon_h \le q_h^{(m)} \le q_{\mathrm{hi}} - \varepsilon_h,
\end{equation}
where \(q_{\mathrm{mid}}\) and \(q_{\mathrm{half}}\) are the midpoint and half-range of each joint's hardware envelope, \(g_m\) is the finger-\(m\) forward kinematics of \Cref{ssec:wristproj}, \(w_{\mathrm{tip}}{=}10^6\) drives a soft equality on the fingertip \(p_i\), and \(\varepsilon_h\) holds every joint at least \(5\%\) of its range inside the URDF limit. The four-joint chain with a three-DoF fingertip leaves a one-DoF kinematic nullspace, and the optimizer slides \(\smash{q_h^{(m)}}\) along it toward \(\smash{q_{\mathrm{mid}}}\). Because LEAP fingers are direct-driven and fully actuated, the four subproblems decouple. On the three holdout objects whose decoded grasps saturated thumb DoFs (the YCB mustard bottle, tennis ball, and cylinder primitive), the minimum hand clearance rises from \(0\%\) of joint range to between \(7\%\) and \(19\%\) with sub-millimeter fingertip drift (worst case \(0.39\,\mathrm{mm}\)), so the wrench-balance set used to construct \(\G\) survives the projection to tactile-noise resolution. A companion arm-side projection refines the arm vector inside the feasible set \(\mathcal{A}\) of \eqref{eq:rb_feas_set} via Drake's \texttt{InverseKinematics}: starting from the retargeted arm we minimize \(\|q_a - \tilde{q}_a\|^2\) subject to \(\|g_{\mathrm{palm}}(q_a) - T_w\|\) within \(5\,\mathrm{mm}/2^\circ\) and a \(\varepsilon_a{=}0.05\,\mathrm{rad}\) margin from each URDF arm limit, with \(q_a\) denoting the arm's joint configuration and \(\tilde{q}_a\) the retargeted arm vector produced by the initial retargeting stage. Because the 6-DoF arm offers no kinematic redundancy to slide along, the regularizer is anchored at \(\tilde{q}_a\) rather than the per-joint midpoint used on the hand side. The projection lifts the worst-case arm-joint clearance from below \(1^\circ\), where our motion planner's joint-padding check rejects the goal outright, to a uniform \(\ge 2.86^\circ\) at a palm-pose cost of \(\le 6\,\mathrm{mm}\), well inside the standoff buffer the pre-grasp phase already carries.

\paragraph{Scene-Aware Generation:}
The refinement in \eqref{eq:hw_refine} accepts any scene description that supplies a hand kinematic model, the world-frame object pose \(\smash{{}^{W}T_O}\) that instantiates \(\smash{\mathcal{A}({}^{O}G;{}^{W}T_O,\mathcal{E})}\), and an environment set \(\mathcal{E}\) of static fixtures. Under that interface, \equidexflow{} emits the candidate \(\G\), the per-finger projection turns it into an executable \(q_h\), and the arm IK yields a \((q_a, q_h)\) pair in \(\mathcal{A}\) without colliding with any fixture in \(\mathcal{E}\). Collision awareness is inherited from whichever fixtures the planner registers in \(\mathcal{E}\), so the same generator transfers to a new tabletop, pedestal, or obstacle set without retraining.

\begin{figure}[htb]
\centering
\setlength{\fboxsep}{0pt}
\setlength{\fboxrule}{0.4pt}
\begin{minipage}[t]{\linewidth}
%% LEFT COLUMN: setup photo and table, stacked
\begin{minipage}[c]{0.38\linewidth}
\centering
\fbox{\includegraphics[
  width=0.8\dimexpr\linewidth-2\fboxrule\relax
]{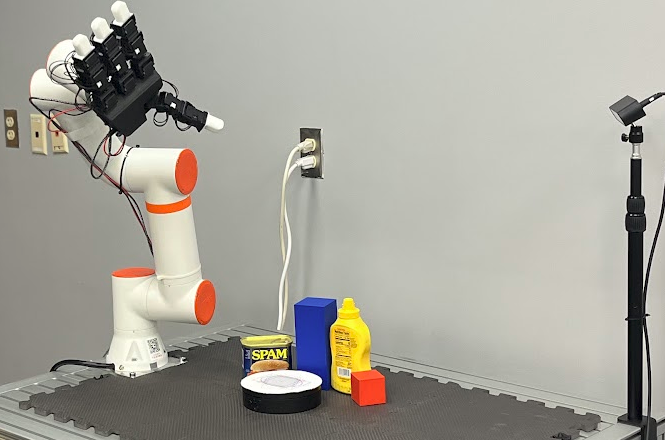}}\\[6pt]
\footnotesize
\setlength{\tabcolsep}{1pt}
\renewcommand{\arraystretch}{0.95}
\captionof{table}{\scshape Retarget Reachability.\\IK Tip = mean fingertip residual.}
\label{tab:hw_reach}
\begin{tabular}[t]{@{}lccc@{}}
\toprule
Object & IK Tip & \(0^\circ\) & \(120^\circ\) \\
       & (mm)   &             &               \\
\midrule
Cube    & 14.3 & Yes & Yes \\
Box     & 14.4 & Yes & Yes \\
Meat Can   & 14.5 & Yes & Yes \\
Mustard & 14.3 & Yes & Yes \\
\midrule
\textbf{Reach.} & & \(\mathbf{4/4}\) & \(\mathbf{4/4}\) \\
\bottomrule
\end{tabular}
\end{minipage}%
\hspace{-10pt}
%% RIGHT COLUMN
\begin{minipage}[c]{0.3\linewidth}
\centering
\setlength{\tabcolsep}{2pt}
\renewcommand{\arraystretch}{0.95}
\begin{tabular}{@{}c@{\hskip 2pt}c@{\hskip 2pt}c@{\hskip 2pt}c@{}}
& {\footnotesize Pre-Grasp}
& {\footnotesize Grasp \& Close}
& {\footnotesize Lift} \\
\multicolumn{4}{@{}l@{}}{\scriptsize\textit{\quad \color{darkgray}Box Primitive}} \\
\rotatebox{90}{\footnotesize \(0^\circ\)}
& \fbox{\includegraphics[width=15mm, ]{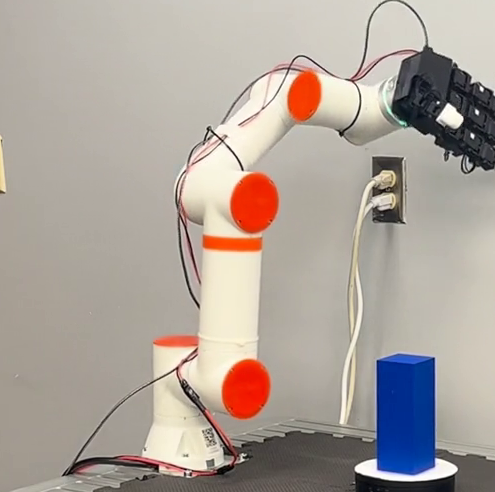}}
& \fbox{\includegraphics[width=15mm, ]{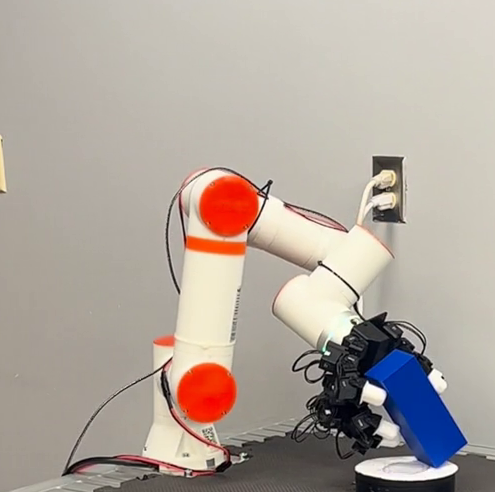}}
& \fbox{\includegraphics[width=15mm, ]{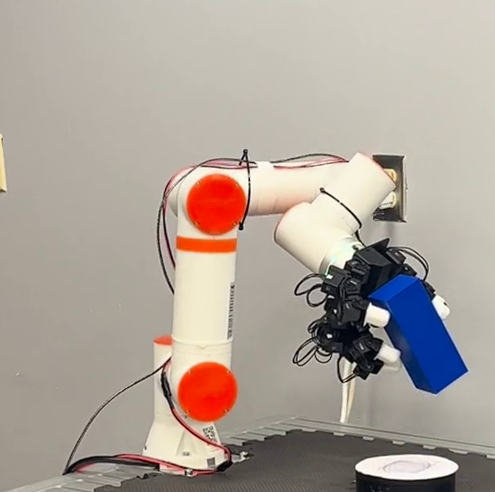}} \\
\rotatebox{90}{\footnotesize \(120^\circ\)}
& \fbox{\includegraphics[width=15mm, ]{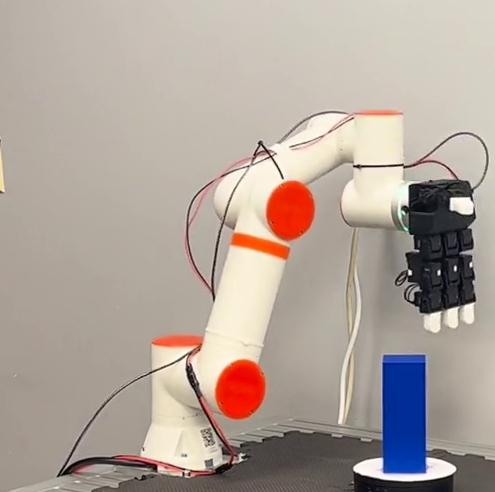}}
& \fbox{\includegraphics[width=15mm, ]{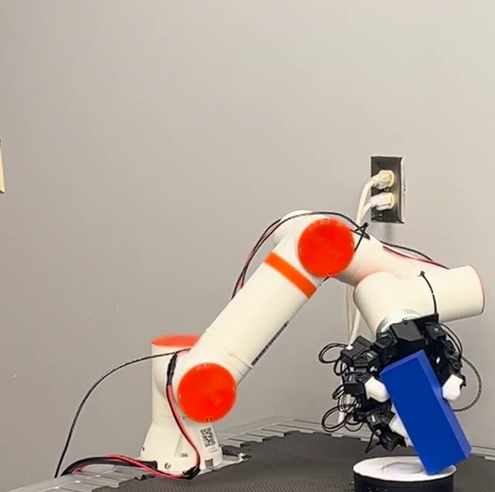}}
& \fbox{\includegraphics[width=15mm, ]{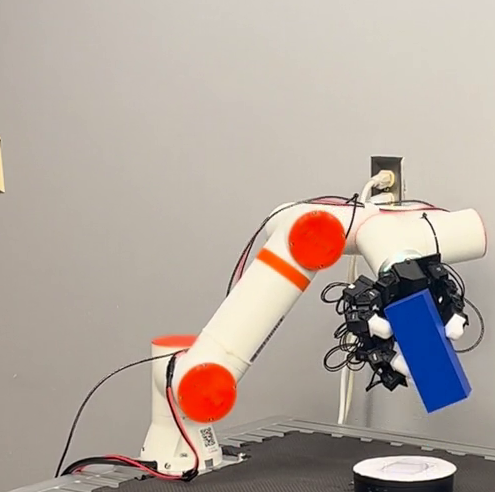}} \\[4pt]
\multicolumn{4}{@{}l@{}}{\scriptsize\textit{\quad \color{darkgray}Potted Meat Can}} \\
\rotatebox{90}{\footnotesize \(0^\circ\)}
& \fbox{\includegraphics[width=15mm, ]{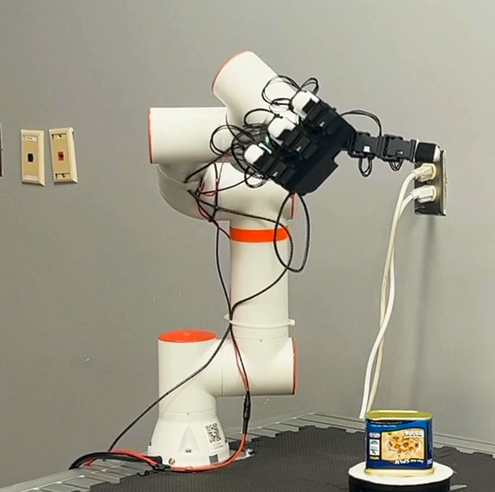}}
& \fbox{\includegraphics[width=15mm, ]{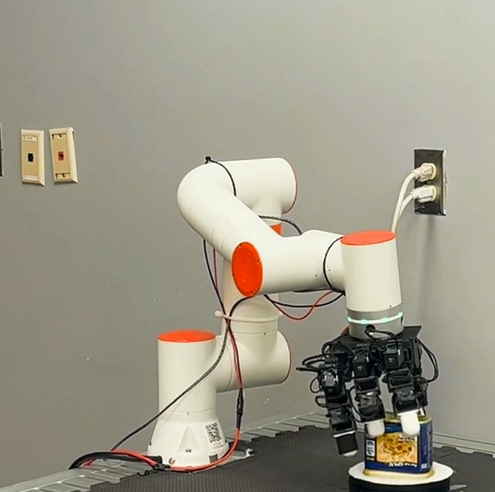}}
& \fbox{\includegraphics[width=15mm, ]{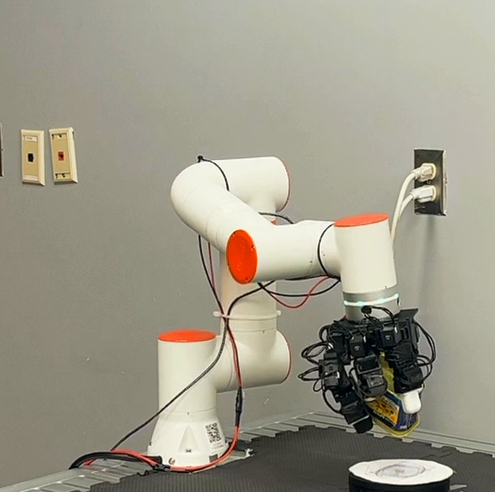}} \\
\rotatebox{90}{\footnotesize \(120^\circ\)}
& \fbox{\includegraphics[width=15mm, ]{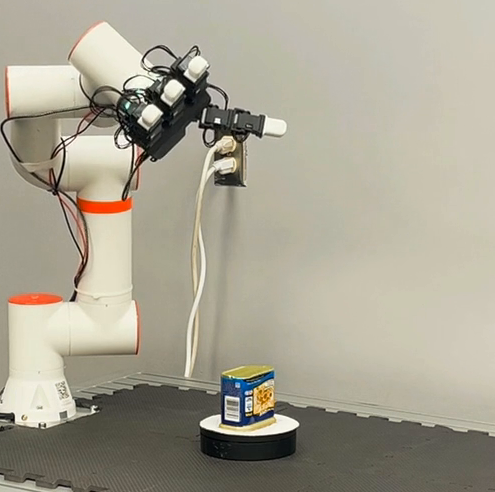}}
& \fbox{\includegraphics[width=15mm, ]{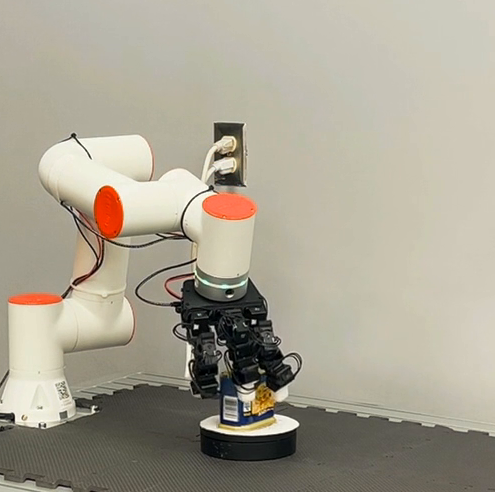}}
& \fbox{\includegraphics[width=15mm, ]{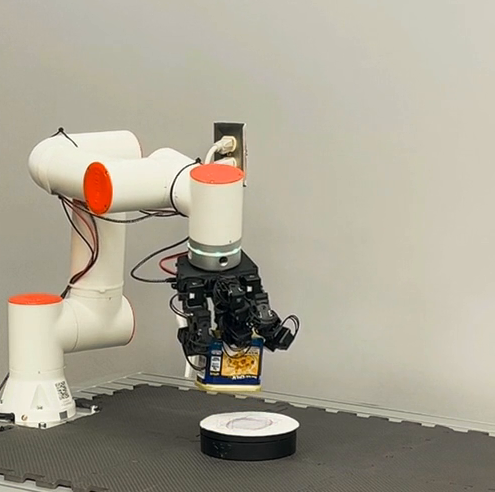}}
\end{tabular}
\end{minipage}
\hfill
\begin{minipage}[c]{0.3\linewidth}
\centering
\setlength{\tabcolsep}{2pt}
\renewcommand{\arraystretch}{0.95}
\begin{tabular}{@{}c@{\hskip 2pt}c@{\hskip 2pt}c@{\hskip 2pt}c@{}}
& {\footnotesize Pre-Grasp}
& {\footnotesize Grasp \& Close}
& {\footnotesize Lift} \\
\multicolumn{4}{@{}l@{}}{\scriptsize\textit{\quad \color{darkgray}Cube}} \\
\rotatebox{90}{\footnotesize \(0^\circ\)}
& \fbox{\includegraphics[width=15mm, ]{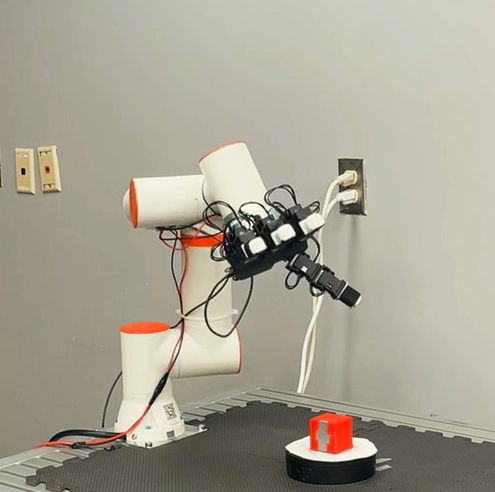}}
& \fbox{\includegraphics[width=15mm, ]{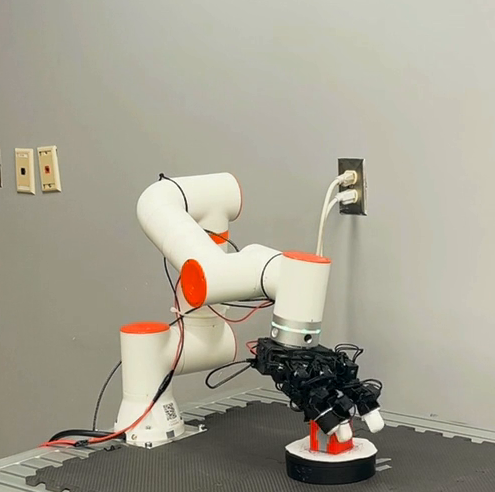}}
& \fbox{\includegraphics[width=15mm, ]{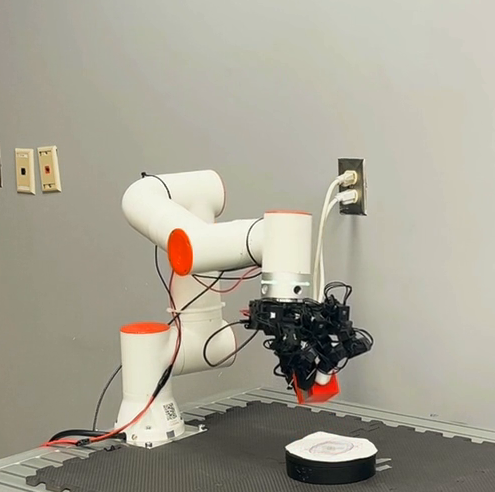}} \\
\rotatebox{90}{\footnotesize \(120^\circ\)}
& \fbox{\includegraphics[width=15mm, ]{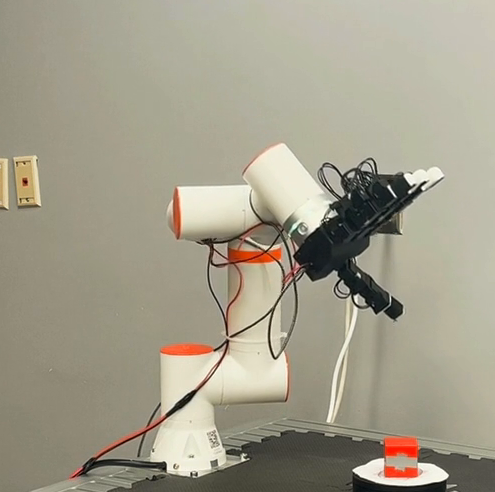}}
& \fbox{\includegraphics[width=15mm, ]{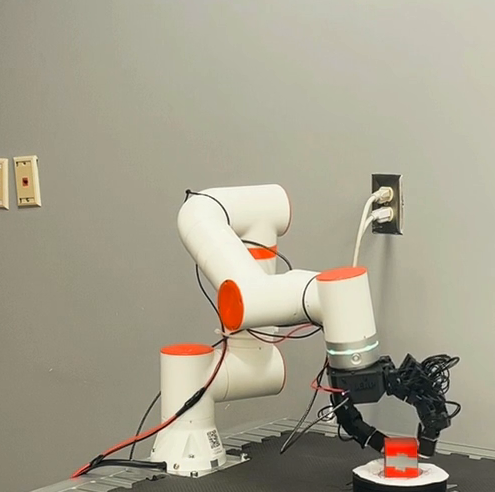}}
& \fbox{\includegraphics[width=15mm, ]{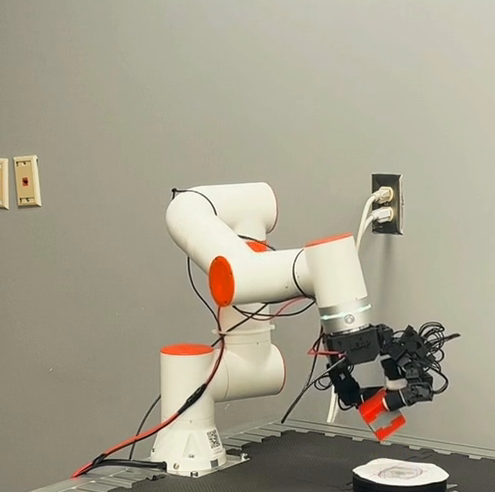}} \\[4pt]
\multicolumn{4}{@{}l@{}}{\scriptsize\textit{\quad \color{darkgray}Mustard Bottle}} \\
\rotatebox{90}{\footnotesize \(0^\circ\)}
& \fbox{\includegraphics[width=15mm, ]{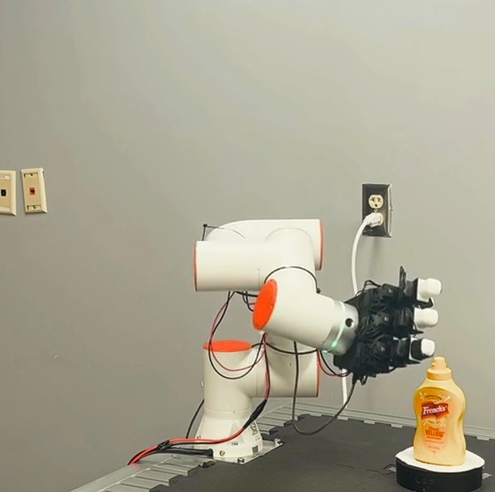}}
& \fbox{\includegraphics[width=15mm, ]{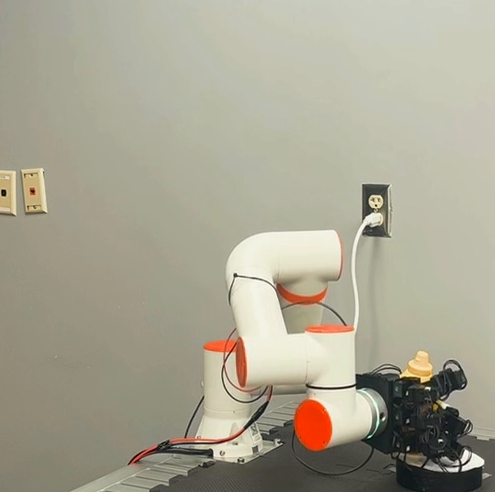}}
& \fbox{\includegraphics[width=15mm, ]{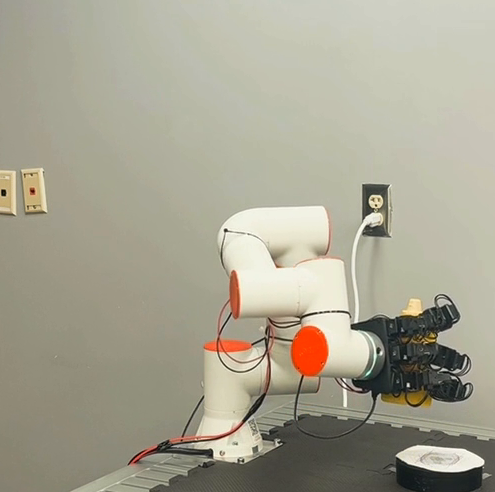}} \\
\rotatebox{90}{\footnotesize \(120^\circ\)}
& \fbox{\includegraphics[width=15mm, ]{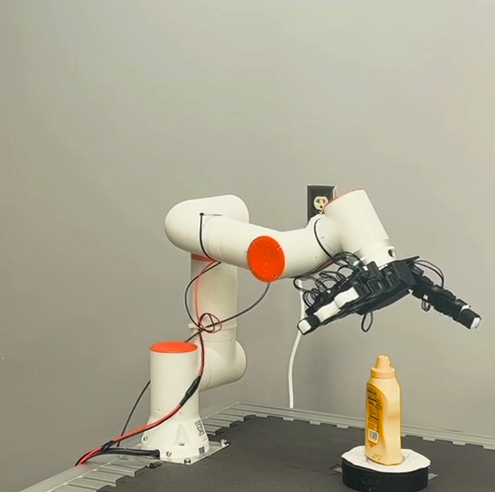}}
& \fbox{\includegraphics[width=15mm, ]{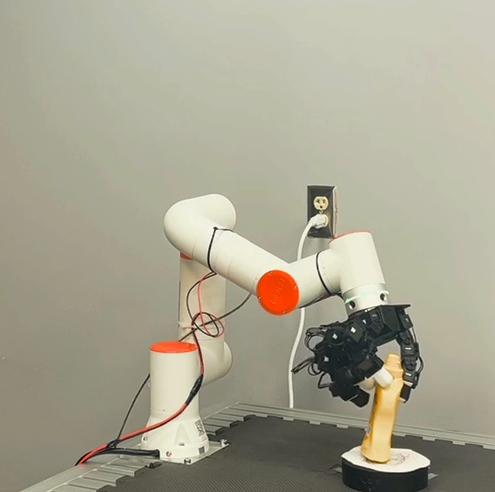}}
& \fbox{\includegraphics[width=15mm, ]{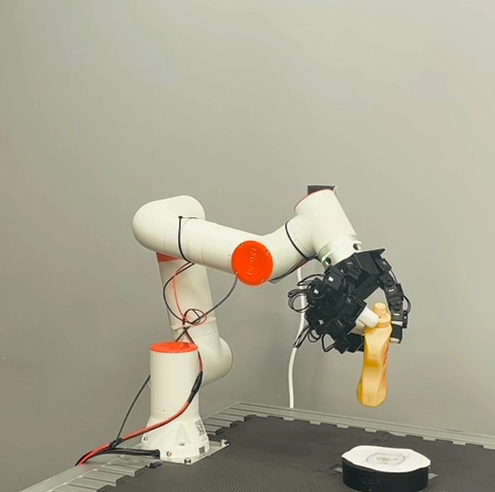}}
\end{tabular}
\end{minipage}
\end{minipage}
\vspace{-1mm}
\caption{\textbf{Hardware Platform and Equivariant Execution.} \emph{Left column, top}: ZArm-mounted LEAP Hand and object-silhouette-marked pedestal (black, tabletop-centered, wide cylinder), with the same decoded and retargeted grasp (\emph{middle and right columns}) executed from pre-grasp through lift at the canonical pose (\(0^\circ\)) and after a \(120^\circ\) object rotation for four representative test objects. The ZED 2i stereo camera in the setup view supplies the single-shot object-pose advisory consumed before each grasp execution (\Cref{sec:supp_perception}).}
\label{fig:hardware}
\end{figure}

\subsection{Limitations}\label{par:limitations}
Several aspects of the design admit principled refinements, which we treat as natural extensions, not fundamental obstacles. The wrist refinement commutes with rigid motions in exact arithmetic (\Cref{apx:eqv_prf}), yet its floating-point implementation re-solves contact IK against the rotated surface and introduces a small orientation-dependent joint perturbation, which an exactly invariant projection operator would remove. The refinement is also a local optimization that can settle in a shallow optimum, and an SE(3)-aware trust-region schedule would harden the convergence. Force-direction fidelity is bounded by centroid-based dataset normals (\Cref{prop:norm_cons}), so depth-sensor normals at inference are a direct improvement. Our formulation targets rigid objects with simulated labels, and the hardware evaluation (\Cref{ssec:hardware}) spans six test objects under open-loop execution. A larger-scale physical success-rate study, deformable objects, and closed-loop manipulation are natural next steps.

\section{Conclusion}\label{sec:conclusion}
Architectural cone projection and surface projection guarantee friction and contact feasibility by construction, while physics-aware losses and a composite ranking close the gap between generative plausibility and manipulation feasibility. Across the 81-object evaluation set, \equidexflow{} attains 0\% friction violations, the best composite physics score among ablations, and a wrist equivariance residual below \(0.04^\circ\). Because the learned contacts and forces live in the object frame, the representation is embodiment-agnostic: we retarget the same generator from the 16-DoF Allegro Hand to a physical LEAP Hand through an IK adapter alone, with per-finger residuals of 5 to 13\,mm, and complete open-loop pick-and-hold trials on all six hardware test objects. Deformable manipulation and closed-loop control are the natural next directions.

\bibliographystyle{unsrt}  
\bibliography{references}

%%%%%%%%%%%%%%%%%%%%%%%%%
% APPENDIX
%%%%%%%%%%%%%%%%%%%%%%%%%
\clearpage
\appendix
\makeatletter
\renewcommand{\thesubsection}{\Roman{subsection}}
\renewcommand{\p@subsection}{\thesection.}
\makeatother
\section*{Appendix}
\section{Architecture Details}\label{apx:arch_details}
The VN-MLP is parameterized by four VN-Linear-LeakyReLU layers (widths in \Cref{tab:hyperparams}), and its input concatenates the 341-channel object feature \(z_O\), the wrist pose as four vector channels (the three rotation-matrix rows and the translation, each a 3-vector), and a lifted scalar time channel obtained by multiplying \(t\) against a learned equivariant basis from \(z_O\), following~\cite{lim2025equigraspflow}. At inference, we integrate from \(T_0\) to \(T_1\) using a Munthe-Kaas fourth-order Runge-Kutta solver~\cite{munthekaas1998rungekutta} on \(\SE(3)\). The Munthe-Kaas scheme maps the current group element to the Lie algebra, performs a classical RK4 step with Lie-bracket corrections, and maps back via the exponential, preserving the group structure at each step. We apply classifier-free guidance~\cite{ho2022cfg} with unconditional drop probability \(p_\mathrm{uncond}\) during training and guidance scale \(\lambda_{\mathrm{cfg}}\) at sampling (\Cref{tab:hyperparams}), yielding the guided velocity field
\begin{equation}
  v_{\mathrm{cfg}} = (1 - \lambda_{\mathrm{cfg}})\,v_{\emptyset} + \lambda_{\mathrm{cfg}}\,v_\theta(T_t, z_O, t).
\end{equation}
\subsection{Contact Decoder}
Given the predicted wrist pose \(\hat{T}_w=(\hat{R}_w,\hat{x}_w)\) and object features \(z_O\), the contact decoder predicts per-finger contact positions. An equivariant path (widths in \Cref{tab:hyperparams}) produces four contact offsets in \(\R^3\) anchored to \(\hat{x}_w\), and a parallel scalar path flattens \(z_O\) and emits four per-finger confidence logits. The raw contacts are then projected onto the object surface via a differentiable soft-nearest-neighbor operation. For each predicted contact \(\hat{c}_i\), \(\hat{c}_i'=\sum_j w_{ij}s_j\) with \(w_{ij}=\mathrm{softmax}_j(-\|\hat{c}_i-s_j\|^2/\tau)\) and \(s_j\) the input surface points. Temperature \(\tau\) (\Cref{tab:hyperparams}) yields near-hard assignment while preserving gradient flow. Surface normals at each projected contact are interpolated by the same weights. Predicted contacts therefore lie on the object surface by construction, eliminating the off-surface drift of free-space regression. The decoder output is \(\hat{C}\in\R^{4\times3}\) (one contact per finger).%

\subsection{Normal Decoder}
The cone projection in the force decoder decomposes each contact force into normal and tangential components relative to an estimated surface normal \(\hat{n}_i\). A naive centroid-to-contact estimate locks force directions to a fixed geometry, shrinking the wrench-balancing feasible set and inflating the wrench residual. To decouple normal estimation from force optimization, a dedicated normal decoder predicts per-finger inward surface normals \(\hat{N}=[\hat{n}_1,\ldots,\hat{n}_4]^\top\in\R^{4\times3}\). For each finger, the decoder concatenates the global features \(z_O\) with the predicted contact \(\hat{c}_i\) (yielding \(\R^{342\times3}\)) and passes the result through a small VN stack (\Cref{tab:hyperparams}) to produce a unit normal. The contact channel provides per-finger spatial grounding. Without it, shared global features make VN layers collapse to collinear outputs. We \(\ell_2\)-normalize the per-finger outputs, and because the path is fully equivariant we have \(\hat{n}_i'=R\hat{n}_i\). Unlike the contact-to-force path, we apply no stop-gradient on the predicted contacts entering the normal decoder, so gradients from \(\mathcal{L}_N\) flow back to the contact decoder and steer contacts toward surface regions with well-defined normals.

\subsection{Joint Configuration Flow}
The flow operates in an unbounded logit space. Each joint angle \(\theta_j\) is mapped to \(x_j=\mathrm{logit}\bigl((\theta_j-\theta_j^{\min})/(\theta_j^{\max}-\theta_j^{\min})\bigr)\), where \([\theta_j^{\min},\theta_j^{\max}]\) are the per-joint kinematic limits of the Allegro Hand. The flow stack uses 8 affine coupling layers with alternating masks (\Cref{tab:hyperparams}), each conditioned on a 128-dim vector fused from the \(\SO(3)\)-invariant features \(\|z_O\|_2\in\R^{341}\) and the wrist pose \(T_w\) (flattened to 12 dims). At training time, the loss is the negative log-likelihood under the change-of-variables formula
\begin{equation}
  \mathcal{L}_q = -\frac{1}{D}\log p_\theta^q(q_h^* \mid z_O, T_w),
\end{equation}
where \(D\) (\Cref{tab:hyperparams}) is the number of joint DoFs and the per-dimension normalization prevents the NLL magnitude from dominating other losses. At inference, stochastic samples are drawn by passing Gaussian noise through the inverse coupling layers and applying the sigmoid map back to joint limits. The hand model is the Allegro Hand (right) with four fingers and joint-angle order \([\theta^{\mathrm{index}}_0, \ldots, \theta^{\mathrm{thumb}}_3]\), where index, middle, ring, and thumb each contribute four revolute joints (one abduction joint and three flexion joints: metacarpophalangeal, proximal, and distal) ordered base-to-tip.

\section{Full Equivariance Proof}\label{apx:eqv_prf}
\begin{proof}[Proof of \Cref{thm:equiv}]
Fix a rigid transform \(A=(R_A,x_A)\in\SE(3)\). It is enough to show that every map in the inference pipeline commutes with \(A\), since equivariant maps are closed under composition. Since the VN-DGCNN encoder satisfies \(z_O(A\mathcal{P})=z_O(\mathcal{P})R_A^\top\) (because VN-Linear layers act only on the channel axis and leave the spatial axis untouched~\cite{deng2021vn}), the SE(3) flow backbone exploits this property by integrating the VN-MLP velocity field \(v_\theta(T_t,z_O,t)\) under this conditioning. Recalling that rotating the conditioning by \(R_A\) and the source sample by \(A\) commutes with the Lie-group integrator on \(\SE(3)\)~\cite{chen2023riemannian,munthekaas1998rungekutta}, we can write  
\(
\Phi_\theta^{1\leftarrow0}(AT_0;\,z_O R_A^\top)=A\,\Phi_\theta^{1\leftarrow0}(T_0;\,z_O).
\)
The contact and normal decoders are pure VN layers conditioned on \(z_O\) and the (now equivariant) wrist, so equivariance propagates through their linear channel-axis operations and gives \(\hat{C}'=R_A\hat{C}+x_A\), \(\hat{N}'=R_A\hat{N}\), and contact frames \(B_i'=R_AB_i\). The force decoder predicts the local coefficients \(\alpha_i\) of \Cref{prop:norm_cons}, which the VN architecture produces as rotation-invariant scalars from equivariant features. Because the cone projection \(\Pi_\mu\) acts entirely in the local contact frame, it commutes with \(R_A\), and \Cref{prop:force_eq} yields \(\hat{f}_i'=R_A\hat{f}_i\). Finally, the wrist refinement \(\phi_{\mathrm{ref}}\) of \Cref{ssec:wristproj} minimizes objective~\eqref{eq:tto_obj}, which assembles only frame-invariant functions of equivariant quantities (forward kinematics, decoded contacts \(\hat{C}\), and the mesh signed distance \(\Phi_M\) with \(\Phi_{AM}(Ap)=\Phi_M(p)\)). Therefore
\[
\phi_{\mathrm{ref}}(A\hat{T}_w,\hat{q}_h,AM)=A\,\phi_{\mathrm{ref}}(\hat{T}_w,\hat{q}_h,M).
\]
Composing these equivariant maps and invoking joint-angle invariance from \Cref{apx:remark_jnt_equv} gives \(\G\sim p_\theta(\cdot\mid\mathcal{P})\Rightarrow A\cdot \G\sim p_\theta(\cdot\mid A\mathcal{P})\), as claimed.
\end{proof}

\section{Approximate Joint Equivariance}\label{apx:remark_jnt_equv}
\Cref{thm:equiv} treats the hand joint vector \(q_h\) as invariant under \(A\). The conditional Real-NVP joint decoder (\Cref{sssec:hand_joint_decoder}) is conditioned on the rotation-invariant features \(\|z_O\|_2\), so in deterministic decoding its output is exactly invariant under input rotation. We confirm this numerically: the per-joint deviation across fresh inferences at each rotation is identically zero across 2000 \(\SO(3)\) samples (\Cref{tab:equivariance_binned}), so the structured output is SE(3)-equivariant in both wrist pose and joint angles up to floating-point precision. The wrist refinement used for penetration-free execution (\Cref{ssec:wristproj}) re-solves contact inverse kinematics against the rotated surface and is the only step that perturbs the joints, applied identically at every orientation.

\section{Grasp Synthesis Pipeline}\label{apx:graspsynth}%
We generate training data using FRoGGeR~\cite{li2023frogger}, a min-weight-metric grasp synthesizer for multi-fingered hands built on Drake~\cite{drake}. FRoGGeR samples initial wrist poses around the object, solves inverse kinematics via Drake's SNOPT, and then runs a nonlinear program that minimizes the worst-case force coefficient required for wrench resistance. The NLP enforces surface contact (equality), joint limits (inequality), collision avoidance (inequality), and optional finger-pad opposition constraints. Each successful grasp stores the 23-dimensional state vector \(q^*\) (the wrist pose as a quaternion and translation, plus the 16 Allegro Hand joint angles), the grasp map \(\graspmap \in \R^{6 \times 3M}\), friction-cone generator matrices \(F_{\mathrm{basis}} \in \R^{3 \times n_s M}\) (\(n_s\) generators per cone, a basis matrix distinct from the Cartesian force tensor \(F\) of \Cref{sec:problem}), LP-optimal force weights \(\alpha_{\mathrm{opt}} \in \R^{n_s M}\), and per-contact positions and normals in both the world and object frames. An adapter converts FRoGGeR output to an intermediate schema for the \equidexflow{} data loader, transforming contact positions to the object frame and extracting per-contact forces via \(f_i = F_{\mathrm{basis},i}\, \alpha_{\mathrm{opt},i}\), where \(F_{\mathrm{basis},i}\) is the \(i\)-th contact's friction-cone basis and \(\alpha_{\mathrm{opt},i}\) collects its LP-optimal weights (linear-cone coefficients, distinct from the local force coordinates \(\alpha_i\) of \Cref{sec:problem}). \Cref{fig:allegro_gallery} shows representative grasps from this pipeline.

\begin{figure}[htb]
\centering
\includegraphics[width=\textwidth]{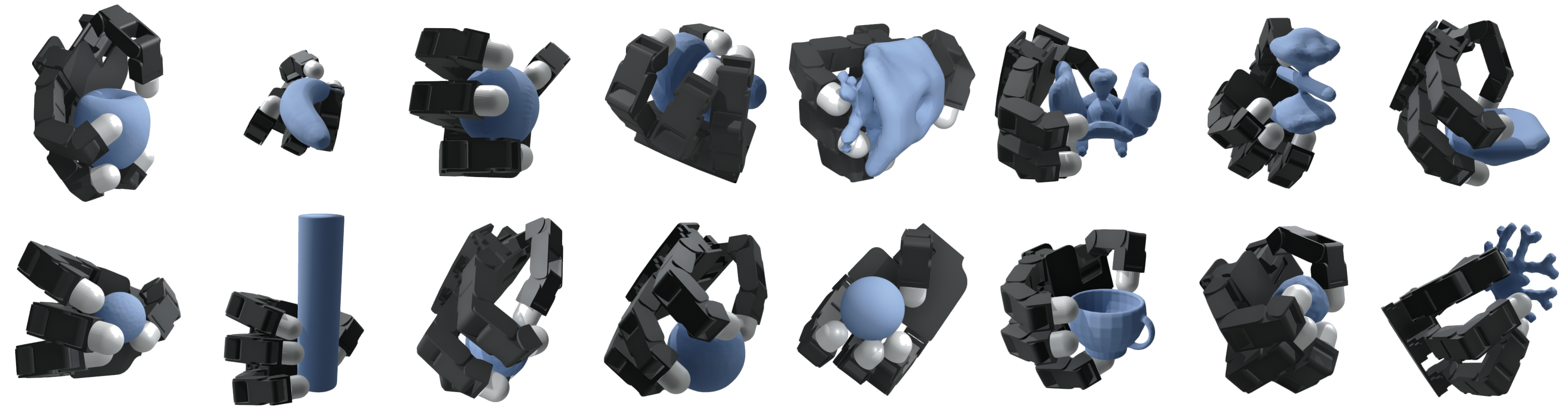}
\caption{\textbf{Allegro Hand Grasps Generated by \equidexflow{}.} Sixteen representative test-set objects spanning EGAD!, YCB, and the cylinder primitive (full inventory in \Cref{apx:dataprov}). Each grasp is decoded from the object point cloud by the \textsc{Full} model and seated on the surface by the wrist refinement of \Cref{ssec:wristproj}.}
\label{fig:allegro_gallery}
\end{figure}

\section{Dexterous Dataset Provenance}\label{apx:dataprov}
We extend the dataset summary of \Cref{sec:experiments} with the full per-record schema. The EGAD source~\cite{morrisonEGADEvolvedGrasping2020b} spans shape complexities 0 to 6 and grasp difficulties A to G. The four geometric primitives are the box, cube, cylinder, and sphere. Per-object generation yields 100 grasps at 100\% synthesis success across the full corpus. Each record stores the object point cloud (512 surface points in the object frame), wrist pose (\(4{\times}4\) SE(3) matrix), joint angles (16-DoF vector \(q\in\R^{16}\)), contact set (4 contact positions and inward surface normals, one per fingertip), contact forces (LP-optimal solutions from FRoGGeR's min-weight formulation under linearized friction cones, with constants in \Cref{tab:hyperparams}), and grasp quality (the \(\varepsilon\)-metric and the min-weight metric \(\ell^*\) from the grasp wrench space).

For analysis, we group the 81 objects into four functional categories by grasp-relevant geometry: \emph{convex} objects approximated by ellipsoids (sphere, cube, fruits, simple EGAD shapes A0 to B6), \emph{prismatic} objects with parallel faces (boxes, cans, EGAD C0 to D6), \emph{cylindrical} objects with a dominant axis of symmetry (bottles, cylinders), and \emph{irregular} objects lacking simple grasp affordances (complex EGAD shapes E0 to G6). This taxonomy, intermediate between coarse three-class schemes~\cite{li2023frogger} and fine-grained 20-class inventories~\cite{aktas_deep_2019}, captures the four distinct contact-placement strategies that the hand must employ. We split the corpus 80\,/\,10\,/\,10 into 6{,}480 training, 809 validation, and 811 test grasps with a fixed random seed (42). During training, we sample a rotation \(R_{\mathrm{aug}}\) uniformly from \(\SO(3)\) and apply it to the point cloud, wrist pose, contact points, contact normals, and forces simultaneously, preserving all geometric relationships.

Each record in \(\dataset\) is a tuple of hand, object, and task parameters, \(\smash{\bigl(\mathcal{P},\allowbreak {}^{O}T_w,\allowbreak q_h,\allowbreak C,\allowbreak \ell,\allowbreak N,\allowbreak F,\allowbreak \alpha,\allowbreak
p_{\mathrm{com}},\allowbreak m,\allowbreak \mu,\allowbreak s\bigr)}\). Here, \(\mathcal{P}\) is the object point cloud, \(\smash{{}^{O}T_w \in \SE(3)}\) is the wrist pose in the object frame, and \(q_h \in \R^{16}\) is the Allegro Hand configuration. The matrices \(C \in \R^{4 \times 3}\) and \(N \in \R^{4 \times 3}\) store contact positions and inward surface normals, respectively. The variable \(F\) denotes the per-contact force vector, \(\alpha\) is the local force coordinate in the contact frame, and \(\ell\) is the per-contact finger label. The quantities \(p_{\mathrm{com}} \in \R^3\) and \(m\) denote the object center of mass and mass in kilograms. The coefficient \(\mu\) is the friction coefficient used during synthesis, and \(s \in \{0,1\}\) is a success flag indicating survivability of the grasp under a bounded exogenous shear force applied at \(p_{\mathrm{com}}\). FRoGGeR's min-weight metric supplies LP-optimal contact forces over linearized friction cones with \(n_s = 4\) generators, \(\mu = 0.7\), and object mass \(m = 0.2\,\mathrm{kg}\). The YCB subset~\cite{calli2015ycb} contributes 28 of the 81 objects.

\section{Hyperparameters}\label{apx:hyperparams}
\Cref{tab:hyperparams} consolidates the architecture and training hyperparameters, and \Cref{tab:loss_components} reports the trained weighted-loss magnitudes for the \textsc{Full} model.

\begin{table}[htb]
\caption{\scshape \equidexflow{} Hyperparameters.}
\label{tab:hyperparams}
\centering
\small
\setlength{\tabcolsep}{1pt}
\renewcommand{\arraystretch}{0.72}
\begin{tabular}{@{}p{0.45\columnwidth} p{0.4\columnwidth}@{}}
\toprule
\textbf{Component} & \textbf{Value} \\
\midrule
\multicolumn{2}{@{}l}{\textit{Encoder (VN-DGCNN)}}\\
\quad VN-Linear widths / \(k\)-NN & \([1,21,21,42,85,170,341]\) / \(40\) \\
\quad Global feature shape & \(\R^{341\times 3}\) \(\bigl(z_O\bigr)\) \\
\midrule
\multicolumn{2}{@{}l}{\textit{SE(3) Flow Backbone}}\\
\quad VN-MLP widths / ODE steps & \([346,128,64,2]\) / MK-RK4 (10) \\
\quad CFG: \(p_{\mathrm{uncond}}\), \(\lambda_{\mathrm{cfg}}\) & \(0.1\), \(2.0\) \\
\midrule
\multicolumn{2}{@{}l}{\textit{Decoder Heads}}\\
\quad Contact VN / scalar MLP & \(341{\to}64{\to}4\) / \(1023{\to}\cdots{\to}4\) \\
\quad Surf.\ proj.\ \(\tau\) / train.\ friction \(\mu\) & \(0.005\) / \(0.5\) \\
\quad Normal / Force VN widths & \(342{\to}64{\to}1\) \\
\quad Joint: coupling layers / cond.\ dim / \(D\) & \(8\) (Real-NVP) / \(128\) / \(16\) \\
\midrule
\multicolumn{2}{@{}l}{\textit{Training}}\\
\quad Optim / lr / wd & Adam~\cite{kingma2015adam} / \(10^{-4}\) / \(10^{-6}\) \\
\quad Batch / epochs / steps per epoch & \(8\) / \(20\) / \(810\) \\
\quad Loss wts & {\raggedright \(1\;(\lambda_{\mathrm{flow}},\lambda_q,\lambda_N,\lambda_F\)),\; \(100\;(\lambda_C,\lambda_{\mathrm{coll}})\),\; \(0.1\;(\lambda_w,\lambda_\mu\))\par} \\
\quad Val / ckpt cadence (steps) & \(200\) / \(500\) \\
\quad GPU & RTX 5070 Ti (16\,GB) \\
\midrule
\multicolumn{2}{@{}l}{\textit{Inference and Ranking}}\\
\quad Candidates \(K\) / rank wts \(\beta_{1{:}4}\) & \(10\) / \(1,\,2,\,1,\,0.5\) \\
\midrule
\multicolumn{2}{@{}l}{\textit{Wrist Refinement (\Cref{ssec:wristproj})}}\\
\quad Steps / lr / \(w_{\mathrm{tr}}\) / \(w_{\mathrm{pen}}\) & \(200\) / \(0.02\) / \(0.05\) / \(50\) \\
\midrule
\multicolumn{2}{@{}l}{\textit{Evaluation}}\\
\quad \(n_s\) / \(\mu\) / mass & \(4\) / \(0.7\) / \(0.2\)\,kg \\
\bottomrule
\end{tabular}
\end{table}

\section{Physics Validation, Grasp Synthesis, and Hardware Execution}\label{apx:supp}
We complement the static contact-quality metrics of \Cref{tab:results} with two \emph{independent} physics-based tests of decoded grasps: a shake test analogous to the GAGrasp force-perturbation-based grasp validation protocol, and a lift test, both implemented in Drake's discrete \texttt{MultibodyPlant} using the Semi-Analytic Primal (SAP) and lagged contact models. In both tests, we sample grasps from \equidexflow{} and pass each grasp's wrist pose and configuration directly to the simulator without retargeting. Below, we summarize each test, focusing on the lift test, and refer the reader to the original GAGrasp paper~\cite{gagrasp} for the shake test's details:

\textbf{Shake (Benchmark):} We adapt the GenDexGrasp/GAGrasp protocol \cite{gagrasp} to Drake with gravity off, applying a \(\pm xyz\) force (\(0.5\,\mathrm{m/s^2}\cdot \mathrm{mass}\), \(\mu=10\)) along all six axes and passing a grasp when the object drifts \(<\SI{2}{cm}\) in every direction. We use this to recover a widely used static-robustness number for comparison. 

\textbf{Lift (This Work):} To assess the physical stability of decoded \equidexflow{} grasps, we add a gravity-enabled test of whether the hand carries the object away and holds it under a lateral disturbance, which is a setting not covered in the shake test. We detail this test next\footnote{Videos available at \href{https://equidexflow.github.io}{\texttt{equidexflow.github.io}}.}.

\subsection{Lift Test Details}\label{ssec:apxlifttest}
We conduct a lift test under lateral perturbations following the protocol exposited below:

\textbf{Mounting and Raise:} We replace the hand's would-be floating base with an actuated prismatic joint along world \(+z\), pinned at the grasp pose so the palm sits at the decoded wrist pose at zero displacement. We use a single DoF because it admits an in-solver implicit-PD actuator, whereas a 6-DoF floating joint does not and an external floating-base PD destabilizes the contact solver. We then raise the object by tracking the smoothstep \(z(t)=H\,(3u^2-2u^3)\), \(u=t/T\), with \(H=\SI{12}{cm}\) and zero endpoint velocity, holding gravity on throughout.

\textbf{Grip from the Model's Predicted Forces:} We drive the grip from the model's own per-finger contact forces rather than an arbitrary closing offset. \equidexflow{} predicts a contact point and a contact force per finger, and we verified from the dataset labels and the decoder implementation that these forces live in the object frame, in newtons, and act compressively along the inward surface normal, sized to balance the object weight quasi-statically. We assign each predicted contact to its nearest fingertip body and apply the feedforward torque \(\tau=\sum_i J_i^{\!\top} f_i\) (with \(J_i\) the contact Jacobian) together with a position-hold PD at the decoded configuration, mapping the object-frame contacts and forces into world through the object pose.

\textbf{Lateral Perturbation and Scoring:} After the raise, we hold the object at full height and ramp a single-axis lateral force along world \(+x\) to each of \(1~\mg\), \(3~\mg\), and \(6~\mg\), easing the ramp so the transient does not fling light objects. We grade a grasp by the largest level at which it retains the object, requiring object-to-hand slip below \SI{3}{cm}, rotation below \(30^\circ\), and at least two contacts throughout. Note that since the synthesis module already penalizes penetration by including a scene-aware pad constraint that bounds pad-to-object penetration to \(\le\SI{1}{mm}\), we do not explicitly score the decoded grasp by penetration.

\textbf{Lift Results:} A Drake protocol that implements the prismatic-Z raise and the \(\{1,3,6\}\mg\) lateral perturbation runs on the LEAP Hand. We defer per-grasp numerical outcomes while we re-tune the closing controller and the fingertip contact stiffness for the LEAP pad geometry. The shake validator of \Cref{ssec:supp_shake_results} is the active hardware-admission gate in this work.

\textbf{Hardware Confirmation:} All six hardware test objects completed an open-loop pick-and-hold on the physical LEAP-ZArm workcell: the symmetric pair (the cylinder primitive and tennis ball) and the asymmetric four (the cube, box primitive, potted meat can, and mustard bottle). The shake validator remains a conservative pre-execution screen, and the deferred Drake lift test above will quantify the post-lift retention margin it approximates.

\subsection{Arm-Aware Grasp Synthesis}\label{sec:supp_armaware}
The hardware reachability study of \Cref{ssec:hardware} finds that the dominant constraint on physical execution is not finger placement but arm reachability: the heuristic initial-condition (IC) sampler inherited from the underlying FRoGGeR generator~\cite{li2023frogger} draws palm poses around the object centroid without reference to the arm's fixed workspace, so a non-negligible fraction of candidates is hand-feasible yet kinematically out of reach for the 6-DoF arm. We extend the FRoGGeR IC pipeline with a reachability filter that wraps the existing heuristic samplers for both the LEAP and Allegro hands. A scene specification encodes the worktable height, the arm-base placement, and a joint-limit-aware reach envelope. A ZArm reachability check screens each candidate world-frame palm pose \(\smash{{}^{W}T_P}\) (palm frame \(P\)) by solving for arm joints that reach the candidate without self-collision or table interference. Rejected poses re-enter the sampler, and the per-finger placement heuristics on the hand side stay untouched. The wrapper is order-preserving: any palm pose the unrestricted sampler would have generated is accepted unless the arm cannot reach it, so the grasp distribution from \equidexflow{} is preserved while the unreachable tail is excised. \Cref{fig:supp_leap_gallery} renders the LEAP Hand at the synthesis solution on twelve objects spanning the YCB and EGAD categories, retargeted from Allegro to LEAP via contact-point IK. We omit the arm and workcell fixtures from this view so the contact patch on each object stays visible. \Cref{tab:supp_leap_sim_results} reports arm reachability separately.

\begin{figure}[htb]
\centering
\setlength{\tabcolsep}{2pt}
\renewcommand{\arraystretch}{0.9}
\newcommand{\gpn}[4]{%
  \begin{tabular}{@{}c@{}}%
    {\footnotesize #2}\\[-1pt]%
    \includegraphics[height=20mm]{figures/cropped/supplementary/leap_gallery/#1.png}\\[-1pt]%
    {\fontsize{5.6}{6.4}\selectfont\sffamily\(\ell^*\)\,#3\,/\,\(\rho_{\mathrm{pen}}\)\,#4}%
  \end{tabular}%
}
\newcommand{\gpnr}[5]{%
  \begin{tabular}{@{}c@{}}%
    {\footnotesize #2}\\[-1pt]%
    \includegraphics[height=20mm, angle=#5]{figures/cropped/supplementary/leap_gallery/#1.png}\\[-1pt]%
    {\fontsize{5.6}{6.4}\selectfont\sffamily\(\ell^*\)\,#3\,/\,\(\rho_{\mathrm{pen}}\)\,#4}%
  \end{tabular}%
}
\begin{tabular}{@{}cccc@{}}
\gpn{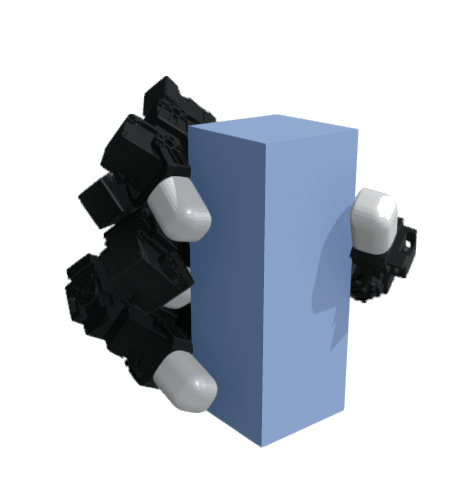}{Box Primitive}{0.024}{0.87}        & \gpn{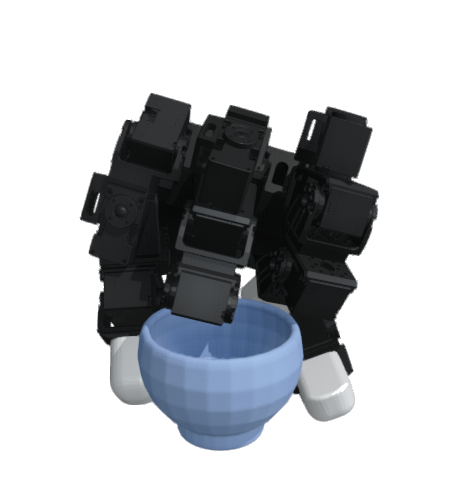}{SNS Cup}{0.040}{0.15}              & \gpn{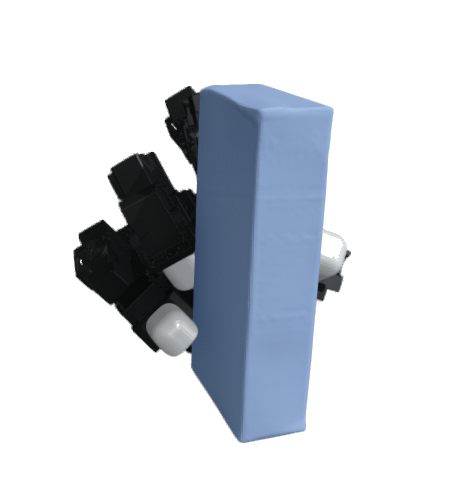}{Sugar Box}{0.025}{0.00}          & \gpn{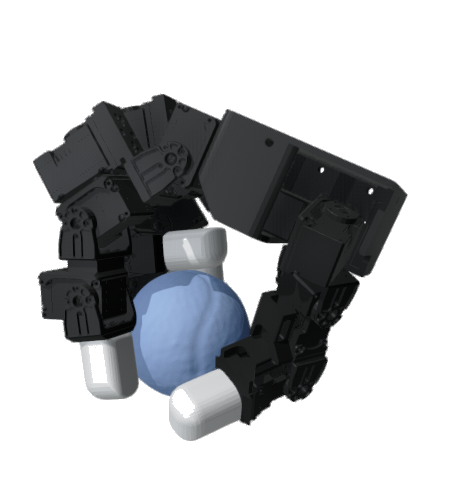}{Tennis Ball}{0.027}{0.32} \\
\gpn{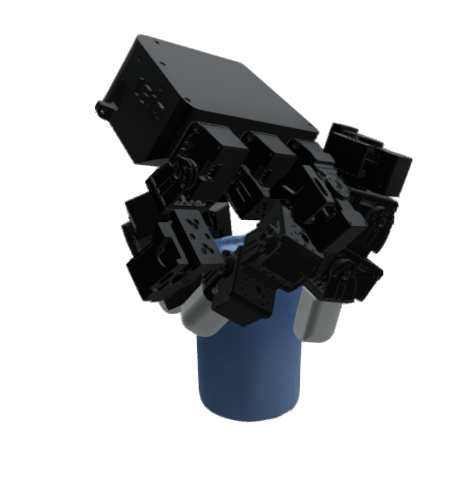}{Tomato Soup Can}{0.019}{0.88} & \gpn{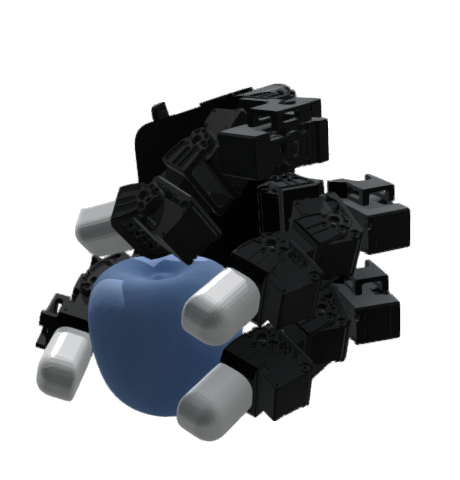}{Apple}{0.026}{0.00}                  & \gpn{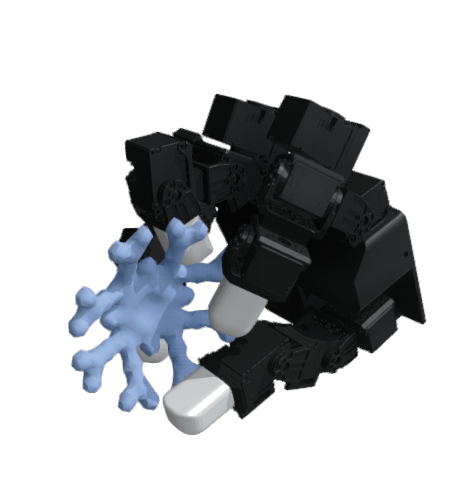}{EGAD! G6}{0.042}{0.00}             & \gpnr{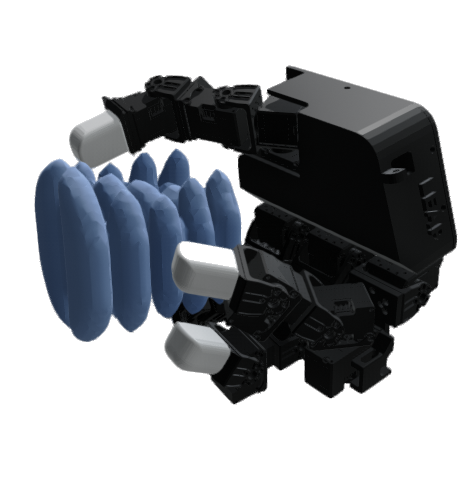}{EGAD! E5}{0.023}{0.00}{90} \\
\gpn{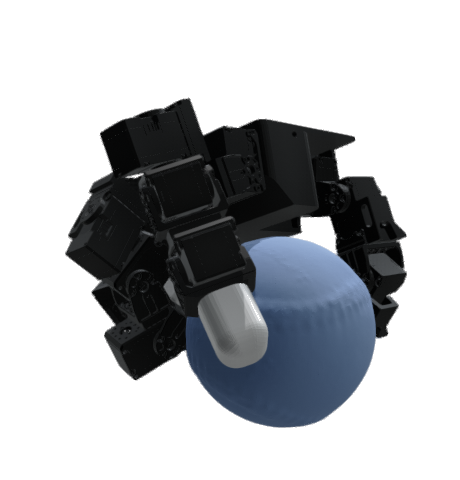}{Softball}{0.029}{0.52}              & \gpn{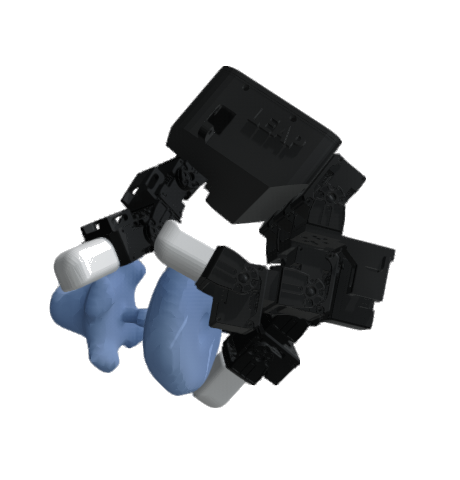}{EGAD! E4}{0.019}{0.00}             & \gpn{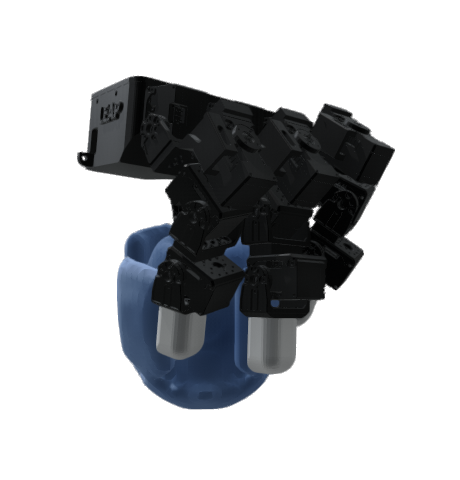}{EGAD! A4}{0.021}{0.54}              & \gpn{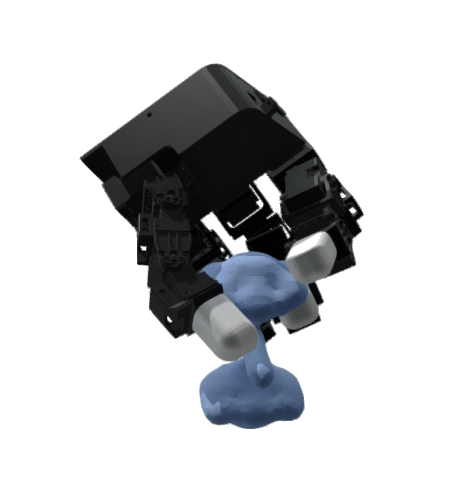}{EGAD! F6}{0.023}{0.31} \\
\end{tabular}
\caption{\textbf{LEAP Grasp Gallery with Per-Grasp Metrics.} Each render shows the LEAP Hand at the synthesis solution after Allegro-to-LEAP contact-point retargeting for one of twelve objects from the primitive, YCB, and EGAD! categories. The two numbers under each panel report (i) the min-weight grasp metric \(\ell^*\)~\cite{li2023frogger} at the synthesis configuration (higher is better, and the FRoGGeR generator maximizes it), and (ii) the worst pad-to-object penetration \(\rho_{\mathrm{pen}}\) (mm) observed across the four fingertips, bounded above by the scene-aware pad constraint at \SI{1}{mm}. The arm and workcell appear in the kinematic-gate table (\Cref{tab:supp_leap_sim_results}), not in this view, so the contact patch on each object stays visible.}
\label{fig:supp_leap_gallery}
\end{figure}

\subsection{Grasp Validation}
Two screens stand between a synthesized grasp and physical realization on hardware: a kinematic lift gate and a physics-based validator.

\textbf{Kinematic Reachability and Lift-Gate Validation:}\label{ssec:supp_lift_gate}
For each grasp accepted by the arm-aware sampler, we re-solve the ZArm IK along a \SI{20}{cm} world-\(+z\) smoothstep chain of 20 waypoints anchored at the synthesis palm pose. A grasp passes the gate when every waypoint is reachable under the scene fixtures of \Cref{sec:supp_armaware}, retains manipulability \(w(q) = \sqrt{\det(J(q)\,J(q)^{\top})} > w_{\min}\), and keeps every arm joint at distance \(\ge \theta_{\min}\) from its limit. A grasp that passes the static FRoGGeR feasibility check can still terminate against a wrist singularity or pin a joint to its stop early in the lift. The gate flags these cases without invoking a physics simulator and excises them before they reach hardware. \Cref{tab:supp_leap_sim_results} reports the minimum manipulability and minimum joint margin attained along the lift trajectory for the 18 validation grasps. All 18 clear the gate. The lift-gate worst case is the master chef can at \(0^\circ\) (\(w_{\min} = 0.010\), \(\theta_{\min} = 22.2^\circ\)). The worst joint margin sits at the shoulder under the can's tall, wide footprint, which is the same failure mode that motivates the gate.

\begin{table}[htb]
\centering
\small
\setlength{\tabcolsep}{6pt}
\renewcommand{\arraystretch}{1.05}
\caption{\scshape Lift-Gate Validation on the Hardware Object Set. Minimum manipulability \(\sqrt{\det(J\,J^{\top})}\) and minimum arm-joint margin observed along the \SI{20}{cm} world-\(+z\) lift, per grasp. The gate accepts a grasp when both quantities exceed their respective thresholds at every waypoint. All 18 grasps clear the gate.}
\label{tab:supp_leap_sim_results}
\begin{tabular}{llcc}
\toprule
Object & Yaw & Min Manipulability & Min Joint Margin (\(^\circ\)) \\
\midrule
Cube             & \(0^\circ\)   & 0.019 & 71.3 \\
                 & \(120^\circ\) & 0.022 & 65.9 \\
Box Primitive    & \(0^\circ\)   & 0.022 & 63.1 \\
                 & \(120^\circ\) & 0.020 & 78.8 \\
Cylinder Primitive & \(0^\circ\)   & 0.019 & 69.9 \\
                 & \(120^\circ\) & 0.016 & 39.4 \\
Master Chef Can  & \(0^\circ\)   & 0.010 & 22.2 \\
                 & \(120^\circ\) & 0.011 & 26.7 \\
Mustard Bottle   & \(0^\circ\)   & 0.017 & 59.3 \\
                 & \(120^\circ\) & 0.015 & 44.7 \\
Potted Meat Can  & \(0^\circ\)   & 0.016 & 62.5 \\
                 & \(120^\circ\) & 0.016 & 60.9 \\
Sphere           & \(0^\circ\)   & 0.016 & 34.5 \\
                 & \(120^\circ\) & 0.021 & 48.4 \\
Tennis Ball      & \(0^\circ\)   & 0.018 & 49.2 \\
                 & \(120^\circ\) & 0.016 & 41.3 \\
Tomato Soup Can  & \(0^\circ\)   & 0.019 & 77.4 \\
                 & \(120^\circ\) & 0.013 & 63.1 \\
\midrule
\textbf{Pass Rate} & & \multicolumn{2}{c}{18 / 18} \\
\bottomrule
\end{tabular}
\end{table}

\textbf{Physics-Based Validator and Hardware Shortlist:}\label{ssec:supp_shake_results}
The kinematic lift gate filters grasps on arm-side feasibility. It says nothing about whether the contact patch survives realistic physics. We close that loop with the GenDexGrasp/GAGrasp shake protocol in Drake under hydroelastic contact, and treat it as a \emph{validator} that shortlists grasps for hardware. A grasp earns admission only if it retains the object under a \(\pm xyz\) force (\(0.5\,\mathrm{m/s^2}\!\cdot\!\mathrm{mass}\), \(\mu=10\)) along all six axes, with worst-axis drift below \SI{2}{cm}. The validator runs the entire 18-grasp set from \Cref{tab:supp_leap_sim_results} and returns a per-grasp verdict. \Cref{tab:supp_leap_shake_results} lists the per-grasp drift, rotation, contact count, and shortlist flag. The validator admits twelve of the eighteen kinematically reachable grasps. The screened-out subset flags configurations whose static FRoGGeR certificate is positive yet whose realized contact patch cannot resist a multi-axis load in simulation, making the validator a conservative complement to the kinematic gate.

\begin{table}[htb]
\centering
\caption{\scshape Physics Validator Output and Hardware Shortlist. Worst-axis object drift and rotation across the six \(\pm xyz\) directions of the GenDexGrasp/GAGrasp protocol, with the minimum contact count maintained during the test. A grasp is admitted when every direction drifts \(<\SI{2}{cm}\). Rows marked ``--'' did not establish a stable contact patch under the validator's load. Twelve of the eighteen kinematically reachable grasps are admitted.}
\label{tab:supp_leap_shake_results}
\small
\setlength{\tabcolsep}{6pt}
\renewcommand{\arraystretch}{1.05}
\begin{tabular}{llcrrc}
\toprule
Object & Yaw & Shortlist & Max Drift (mm) & Max Rot (\(^\circ\)) & Contacts \\
\midrule
Cube             & \(0^\circ\)   & \ding{55}      & --     & --    & 0 \\
                 & \(120^\circ\) & \ding{51}      & 0.4    & 0.3   & 4 \\
Box Primitive    & \(0^\circ\)   & \ding{51}      & 2.4    & 2.4   & 4 \\
                 & \(120^\circ\) & \ding{51}      & 9.9    & 8.7   & 3 \\
Cylinder Primitive & \(0^\circ\)   & \ding{51}      & 1.4    & 2.0   & 3 \\
                 & \(120^\circ\) & \ding{51}      & 1.3    & 2.6   & 2 \\
Master Chef Can  & \(0^\circ\)   & \ding{51}      & 0.8    & 0.7   & 3 \\
                 & \(120^\circ\) & \ding{55}      & 43.1   & 23.9  & 2 \\
Mustard Bottle   & \(0^\circ\)   & \ding{51}      & 3.2    & 2.8   & 3 \\
                 & \(120^\circ\) & \ding{51}      & 3.4    & 3.3   & 4 \\
Potted Meat Can  & \(0^\circ\)   & \ding{51}      & 0.9    & 0.9   & 4 \\
                 & \(120^\circ\) & \ding{51}      & 9.2    & 8.3   & 4 \\
Sphere           & \(0^\circ\)   & \ding{51}      & 1.4    & 1.4   & 3 \\
                 & \(120^\circ\) & \ding{55}      & --     & --    & 0 \\
Tennis Ball      & \(0^\circ\)   & \ding{55}      & 56.3   & 58.1  & 2 \\
                 & \(120^\circ\) & \ding{51}      & 8.7    & 9.3   & 3 \\
Tomato Soup Can  & \(0^\circ\)   & \ding{55}      & --     & --    & 0 \\
                 & \(120^\circ\) & \ding{55}      & --     & --    & 0 \\
\midrule
\textbf{Hardware Shortlist} & & \multicolumn{4}{c}{12 / 18 admitted} \\
\bottomrule
\end{tabular}

\end{table}

\subsection{Equivariant Allegro-to-LEAP Retargeting}\label{sec:supp_retarget}
The hardware experiments of \Cref{ssec:hardware} execute \equidexflow{} on the LEAP Hand by retargeting decoded Allegro fingertip contacts through LEAP inverse kinematics. We document the per-finger correspondence and the residual error budget that underlie this step, neither of which is detailed in \Cref{ssec:hardware}.

\textbf{Finger Correspondence:} The retarget treats each fingertip as an end-effector with the contact target expressed in the object frame. Allegro's \texttt{algr\_rh\_\{if,mf,rf,th\}\_ds} distal links are matched to LEAP's \texttt{fingertip}, \texttt{fingertip\_2}, \texttt{fingertip\_3}, and \texttt{thumb\_fingertip} respectively, with the wrist pose held fixed at the value selected by the arm-aware sampler of \Cref{sec:supp_armaware}.

\textbf{Residual Error Budget:} \Cref{tab:supp_retarget_residual} reports the per-finger tip residual on the \equidexflow{} object set. A naive joint-space copy that matches Allegro's four-DoF-per-finger joints to their nearest LEAP analogues leaves a \SIrange{29}{54}{\milli\meter} residual at the fingertips, because the two hands' phalanx lengths and abduction axes diverge. Per-finger IK closes this gap to \SIrange{5}{13}{\milli\meter}, consistent with the \SIrange{14.3}{14.5}{\milli\meter} mean residuals reported in the hardware reachability table, where the additional millimeters reflect contact-IK standoff against the Ninjaflex pad surface.

\begin{table}[htb]
\centering
\small
\setlength{\tabcolsep}{8pt}
\caption{\scshape Allegro-to-LEAP Retarget Residuals. Fingertip residual ranges across the \equidexflow{} object set. The joint-space copy serves as a sanity baseline. The LEAP grasps executed in the hardware-experiments section use per-finger IK.}
\label{tab:supp_retarget_residual}
\begin{tabular}{lc}
\toprule
Retarget Strategy & Mean Fingertip Residual (mm) \\
\midrule
Naive Joint Copy   & \numrange{29}{54} \\
Per-Finger IK      & \numrange{5}{13}  \\
\bottomrule
\end{tabular}
\end{table}

\textbf{Equivariance \& Grasp Retargeting:} The retargeting procedure is not \(\SO(3)\)-equivariant by construction. It solves IK against fingertip positions whose equivariance is inherited from the upstream \equidexflow{} decoder, and the LEAP kinematic chain is non-isomorphic to Allegro's, so the on-hand equivariance residual is bounded by the synthesis-side residual of \Cref{ssec:eq_guarantee} plus an IK-induced term. We defer a native LEAP equivariance study to future work, consistent with the limitations in \Cref{par:limitations}.

\subsection{Perception-in-the-Loop Execution}\label{sec:supp_perception} Our hardware executor consumes an externally supplied object pose rather than relying on a hand-placed object. We store a per-grasp perception block alongside the decoded grasp metadata so the executor can recompute the palm target at runtime from a perceived pose without rerunning fingertip contacts or arm trajectories. Let \(\smash{{}^{W}T_O^{q^\star}}\) and \(\smash{{}^{W}T_P^{q^\star}}\) denote the world-frame object and palm (frame \(P\)) poses at the synthesis solution \(q^\star\). We precompute and store the rigid offset
\begin{equation}
{}^{O}T_P \;:=\; \bigl({}^{W}T_O^{q^\star}\bigr)^{-1} \, {}^{W}T_P^{q^\star},
\end{equation}
so the perceived palm target at execution time is \(\smash{{}^{W}T_P^{\,\mathrm{perc}} \;=\; {}^{W}T_O^{\,\mathrm{perc}} \, {}^{O}T_P,}\)
and the IK target on the arm tool flange is \(\smash{{}^{W}T_{\mathrm{tool0}} \;=\; {}^{W}T_P^{\,\mathrm{perc}} \, \bigl({}^{P}T_{\mathrm{tool0}}\bigr)^{-1},}\)
with \(\smash{{}^{P}T_{\mathrm{tool0}}}\) the calibrated palm-to-flange offset. The arm-constrained execution problem is separate from the object-relative generation pipeline. For an estimated object pose \(\smash{{}^{W}T_O}\), a sampled relative grasp induces the wrist target \(\smash{{}^{W}T_w = {}^{W}T_O\,{}^{O}T_w}.\) An executable grasp must further satisfy workspace and configuration constraints, which we capture with the set
\begin{multline}
\mathcal{A}({}^{O}G;{}^{W}T_O,\mathcal{E})
= \bigl\{(q_a,q_h) \bigm| \mathrm{FK}_{\mathrm{arm}}(q_a) = {}^{W}T_O\,{}^{O}T_w,
 q_a \in \mathcal{Q}_a,\; q_h \in \mathcal{Q}_h,\; \mathrm{Coll}(q_a,q_h,\mathcal{E}) = 0 \bigr\},
\label{eq:rb_feas_set}
\end{multline}
where \(\mathcal{E}\) contains the robot base, table, obstacles, joint limits, and approach constraints. Arm-aware execution is therefore not SE(3)-equivariant, because the robot base and environment define a fixed world frame.

\textbf{Pose Source:} Our hardware executor is agnostic to the provenance of \(\smash{{}^{W}T_O^{\,\mathrm{perc}}}\). A fiducial-aligned target or a single-shot pose estimator both satisfy the interface, since the perception pipeline serves as a single-shot pose advisory before grasp execution rather than closed-loop visual servoing, which we defer to future work.

\textbf{Close Strategies:} To close the fingers and establish the grasp, we select hand-closure timing per object class. Soft and compliant objects (e.g., the mustard bottle and the tennis ball) use a during-approach closure, in which the finger joints interpolate toward the decoded configuration concurrently with the final approach-path segment, exploiting deformation to absorb residual pose error. Rigid objects (cube, box primitive, potted meat can, cylinder primitive) use an after-arrival closure that holds the open pre-grasp configuration until the palm reaches the target and only then commands the decoded joints, avoiding premature contacts that would displace the object.

\subsection{Hardware Object Set and Execution Protocol}\label{sec:supp_hw_protocol}
\Cref{tab:supp_objects} lists the six objects used in our hardware experiments, their source, the rotational symmetry class, the angles evaluated, and the close strategy. We evaluate asymmetric objects (mustard bottle, cube, box primitive, potted meat can) at \(0^\circ\) and \(120^\circ\) about the vertical axis (\Cref{fig:pedestal_jig}), matching the convention of the hardware reachability table. We evaluate the symmetric objects (cylinder primitive, tennis ball) at the canonical pose only: grasps under any \(\SO(2)\)-axisymmetric or \(\SO(3)\)-isotropic rotation are pose-equivalent, so a second angle does not exercise the executor's equivariance handling. We class the cube as asymmetric under the \(120^\circ\) test because its rotational symmetry group is generated by \(90^\circ\) rotations, and a \(120^\circ\) rotation yields a distinct pose. The executor protocol omits the cylinder primitive from rotation testing on the same symmetry grounds.

\begin{figure}[htb]
\centering
\includegraphics[width=.8\columnwidth]{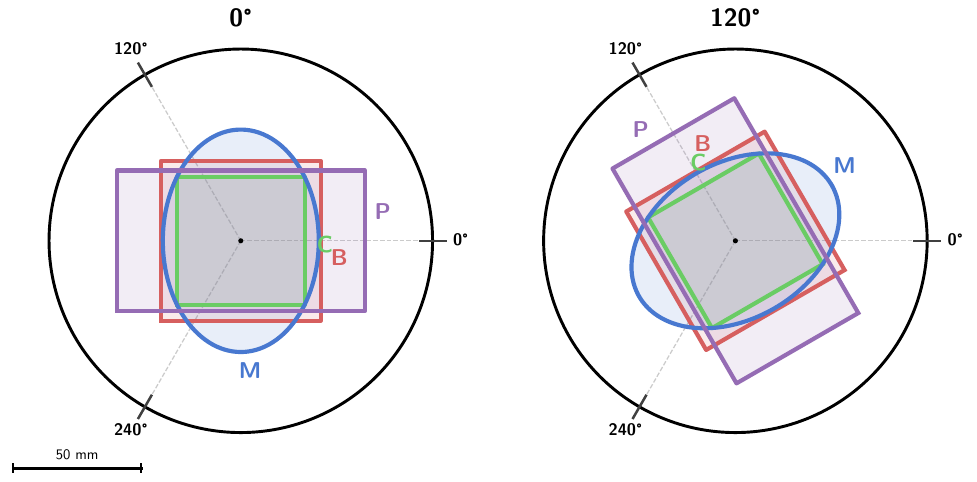}
\caption{\textbf{Pedestal Registration Jig.} Object footprints on the \SI{150}{mm}-diameter cylindrical pedestal at the canonical pose (\emph{left}, \(0^\circ\)) and after a \(120^\circ\) rotation about the vertical axis (\emph{right}). The four outlines co-rotate rigidly between panels, and \(B/C/M/P\) denote the box primitive, cube, mustard bottle, and potted meat can. Printed at full scale and affixed to the pedestal visible in \Cref{fig:hardware}, the jig fixes repeatable object placement at both poses. The symmetric objects (cylinder primitive, tennis ball) are placed at the canonical pose only and are thus omitted.}
\label{fig:pedestal_jig}
\end{figure}

\begin{table}[htb]
\centering
\small
\caption{\scshape Hardware Object Set. Symmetric objects are evaluated at the canonical pose only. Asymmetric objects are evaluated at \(0^\circ\) and \(120^\circ\) about the vertical axis, matching the convention of the hardware reachability table.}
\label{tab:supp_objects}
\setlength{\tabcolsep}{4pt}
\renewcommand{\arraystretch}{0.85}
\begin{tabular}{llllc}
\toprule
Object & Source & Symmetry & Angles & Close Strategy \\
\midrule
Mustard Bottle    & YCB 006    & Asymmetric                & \(0^\circ,\,120^\circ\) & During Approach \\
Cube              & Synthetic  & Asymmetric                & \(0^\circ,\,120^\circ\) & After Arrival   \\
Box Primitive     & GraspIt!   & Asymmetric                & \(0^\circ,\,120^\circ\) & After Arrival   \\
Potted Meat Can   & YCB 010    & Asymmetric                & \(0^\circ,\,120^\circ\) & After Arrival   \\
Cylinder Primitive & GraspIt!   & Axisymmetric \(\SO(2)\)   & Canonical               & After Arrival   \\
Tennis Ball       & YCB 056    & Isotropic \(\SO(3)\)      & Canonical               & During Approach \\
\bottomrule
\end{tabular}
\end{table}

\textbf{Open-Loop Execution Protocol:} The executor runs an open-loop pre-grasp, approach, close, lift, hold sequence consistent with prior dexterous-grasping hardware studies~\cite{li2023frogger,wang2023dexgraspnet,zhang_dexgraspnet_2024}. The arm moves to a pre-grasp pose obtained by retracting the perceived palm target \SI{5}{cm} along the palm approach axis, with the hand pre-shaped to the decoded open configuration. The approach segment linearly interpolates the wrist to the palm target. The close stage commands the decoded joint configuration under the object-class strategy in \Cref{sec:supp_perception}. The lift stage raises the arm by \SI{15}{cm} along world \(+z\), and a \SI{3}{s} hold concludes the trial. We tune per-segment durations and an optional inward squeeze offset per object.

\textbf{Success Criterion:} We record a trial as a success if the object remains stably grasped, with no slip or drop, through the \SI{3}{s} hold after lift. The object dropping during lift or hold counts as a failure. We report a partial-success category, in which the object lifts but slips during the hold, separately when relevant.

\textbf{Results Summary from Hardware Trials:} The hardware reachability table (\Cref{tab:hw_reach}) is a pre-execution, computational result: it reports which retargeted grasps clear the 6-DoF arm's workspace at each orientation (\(4/4\) at both \(0^\circ\) and \(120^\circ\)). The trials here are its post-execution complement: which of those predicted-reachable grasps realized a stable pick on our hardware platform. All four asymmetric objects (the cube, box primitive, potted meat can, and mustard bottle) completed the full pre-grasp, approach, close, lift, and \SI{3}{s} hold without slip or drop at both the canonical and \(120^\circ\) vertical-axis poses, and the symmetric cylinder primitive and tennis ball completed the same sequence at the canonical pose, for ten successful trials across the six objects. That every asymmetric object retains the object at both orientations is the hardware counterpart to the equivariant execution analyzed in \Cref{ssec:hardware}: the canonical grasp and its \(120^\circ\) co-transform both complete the pick and the hold. The trials thus realize the full reachable set predicted in \Cref{tab:hw_reach}.

\textbf{Known Limitations:} Execution is open-loop and uses single-shot perception correction rather than closed-loop visual servoing. The LEAP Hand's Ninjaflex pads can edge-load on small or sharply curved surfaces, and we chose the object set to remain within the hand's effective grasp aperture. We enforce equivariance on the synthesis side and inherit it at execution through the retarget. We defer quantifying on-hand recovery of the synthesis-side equivariance residual (\Cref{ssec:eq_guarantee}) to future work.

\end{document}